\documentclass{article}

\PassOptionsToPackage{numbers, compress}{natbib}

\usepackage[utf8]{inputenc}
\usepackage[T1]{fontenc}
\usepackage{lmodern}
\usepackage{natbib}
\usepackage{geometry}
\usepackage{hyperref}
\hypersetup{
  pdftitle={Fast Learning Rates for Physics-Informed Kernel Methods},
  pdfauthor={Luc Brogat-Motte, Joachim Bona-Pellissier, Giacomo Meanti, Lorenzo Rosasco}
}
\usepackage{url}
\usepackage{booktabs}
\usepackage{amsfonts}
\usepackage{nicefrac}
\usepackage{microtype}
\usepackage{xcolor}
\usepackage{graphbox}
\usepackage{bm}

\definecolor{todoblue}{RGB}{40,95,170}

\usepackage{amsmath}
\usepackage{amssymb}
\usepackage{stmaryrd}
\usepackage{dsfont}
\usepackage{graphicx}
\usepackage{amsthm}
\newtheorem{assumption}{Assumption}

\newcommand{\R}{\mathbb{R}}

\newcommand{\E}{\mathbb{E}}

\newcommand{\bmH}{\mathcal{H}}

\newcommand{\bmX}{\mathcal{X}}

\DeclareMathOperator{\Tr}{Tr}

\newenvironment{sproof}{%
\proof}{\endproof}

\usepackage{thmtools, thm-restate}
\usepackage{graphicx}
\newtheorem{theorem}{Theorem}[section]

\newtheorem{corollary}{Corollary}[section]
\newtheorem{lemma}[theorem]{Lemma}
\newtheorem{proposition}[theorem]{Proposition}

\newcommand{\cH}{\mathcal{H}}
\newcommand{\cX}{\mathcal{X}}
\newcommand{\RR}{\mathbb{R}}
\newcommand{\EE}{\mathbb{E}}

\DeclareMathOperator{\Var}{Var}

\newcommand{\dop}{D}

\usepackage[capitalise]{cleveref}
\crefname{assumption}{assumption}{assumptions}

\title{Fast Learning Rates for Physics-Informed Kernel Methods}

\author{%
  Luc Brogat-Motte$^{1}$, Joachim Bona-Pellissier$^{2}$, \\
  Giacomo Meanti$^{2}$, Lorenzo Rosasco$^{1,2}$\\[0.6em]
  \small $^1$SAIL Unit, Istituto Italiano di Tecnologia, Genoa, Italy\\
  \small $^2$MaLGa Center, DIBRIS, Università degli Studi di Genova,
    Genoa, Italy\\
  \small \texttt{luc.brogatmotte@iit.it}, \texttt{joachim.bona@edu.unige.it},
    \\
  \small \texttt{giacomo.meanti@edu.unige.it},
    \texttt{lorenzo.rosasco@unige.it}%
}

\date{}

\begin{document}

\maketitle

\begin{abstract}

In physics-informed machine learning, a target function \(u^*\) is learned from noisy value observations \(y_i=u^*(x_i)+ \varepsilon_i\), together with differential information, given either by noisy observations \(d_j=(Du^*)(z_j)+\xi_j\) or by a known physical constraint \(Du^*=v\). We consider the setting where \(D\) is a linear differential operator and analyze a physics-informed kernel estimator \(\hat u\)  combining \(n\) value observations and \(m\) differential observations. In this context, we ask how much can differential information improve predictions, and how does this improvement depend quantitatively on \(n\), \(m\), and \(D\).  We prove finite-sample bounds, supported by numerical simulations, revealing a two-regime structure for the prediction error. When \(m\) is limited, the rate depends jointly on \(n\) and \(m\); when \(m\) exceeds a problem-dependent threshold, the rate saturates and matches the oracle rate obtained when the perfect constraint \(D \hat u = Du^*\) is imposed.  
Examples are discussed for Sobolev spaces which are reproducing kernel Hilbert spaces and  include partial Laplacian constraints on the torus and gradient observations on bounded domains. These examples illustrate the range of possible learning rate improvements -- from the standard nonparametric $n^{-1/4}$ to the parametric rate $n^{-1/2}$. 
Finally, we derive physically consistent rates in a stronger norm that jointly controls the errors  in \(\hat u\) and \(D\hat u\).

\end{abstract}

\section{Introduction}

Supervised learning algorithms seek to estimate an unknown target function from noisy observations of its values. In many applications arising in physics, engineering, and scientific computing, however, one has access not only to value observations, but also to physical information involving derivatives or other differential quantities associated with the target function~\cite{cuomo2022scientific}.
This additional information is often specified through functional constraints, such as gradients, Partial Differential Equation (PDE) residuals, or conservation laws \citep{raissi2019physics, karniadakis2021physics, rackauckas2020universal}. 
Concrete examples include learning interatomic potentials, where energies and forces of atomic systems provide coupled value and gradient information through \(F=-\nabla E\)~\citep{behler07potentials,bartok10gap}; 3D reconstruction problems, where surface normals give differential information about the underlying shape~\citep{amos20surface,patel25surface}; or computational cardiology, where physical models involve differential constraints from fluid dynamics~\citep{quarteroni25general,botarelli25navier}.

In this paper we investigate physics-informed machine learning from the point of view of statistical learning theory. While there exists a comprehensive theory describing the statistical performance of learning from  value observations \cite{caponnetto2007optimal, gyorfi2002distribution}, much less is known about how differential information affects learning rates and sample complexity. We are particularly interested in  quantitative results of how differential information improves prediction: what is the improvement it can provide?  how many differential observations are needed for this gain to appear?

We study this question for target functions  in a reproducing kernel Hilbert space (RKHS) \(\mathcal H\), which makes precise analyses feasible.
We assume that there exists a target function \(u^* \in \mathcal H\) and that, in addition to noisy value observations
\[
y_i = u^*(x_i) + \varepsilon_i, \qquad i=1, \dots, n,
\]
also noisy differential quantities are available
\[
d_j = (Du^*)(z_j) + \xi_j, \qquad j=1, \dots, m,
\]
where \(D\) is a linear differential operator.

Our contributions are as follows:
\begin{itemize}
    \item We derive finite-sample bounds that quantitatively describe how the prediction error depends on sample sizes \(n,m\), differential operator \(D\), and regularization parameters \(\lambda,\gamma\).

    \item We introduce a novel capacity assumption for physics-informed learning. This assumption decomposes the learning problem into a component that can be learned from differential data and a residual component not seen by \(D\). It then quantifies \emph{how fast} the first component can be learned from physics data and \emph{how much} complexity remains in the residual one. We illustrate both quantities with two interpretable Sobolev RKHS examples, showing how they depend on the choice of \(D\).

    \item Under this assumption, we obtain learning rates by deriving the best choice of  \(\lambda\) and \(\gamma\). The rates exhibit two regimes: a differential-limited regime, where the error decreases jointly with \(n\) and \(m\), and a saturation regime, where the rate matches the one obtained when the physical constraint is known exactly. In the Sobolev space examples, the saturated prediction rates range from \(n^{-1/4}\) to the parametric rate \(n^{-1/2}\). Simulations further illustrate the predicted improvements and two-regime behavior.
    
    \item Finally, we show that the same estimator is physically consistent beyond value prediction. We prove learning rates in a stronger norm that jointly controls the errors in \(\hat u\) and \(D\hat u\), ruling out accurate value prediction without learning the differential structure.
\end{itemize}

\subsection{Related work}

\textbf{Settings.~~}
Learning from both function values and derivatives (or, more generally, differential operators) is classically known as \textit{Hermite-Birkhoff} interpolation, whose study in a statistical setting dates back at least to the seminal work of \citet{kimeldorf1971some}. This setting has recently resurfaced in machine learning under the names Sobolev training~\citep{czarnecki2017sobolev} and physics-informed learning~\citep{raissi2017physics}. 
Incorporating differential information, whether in discrete or continuous form, can serve various purposes. Differential operators can be used to regularize the learning procedure: penalties of the form $\int (Du(x))^2 dx$ have been well studied for spline smoothing~\citep{wahba1990spline}, inverse problems~\citep{hanke1992regularization,engl1996regularization}, and machine learning~\citep{poggio90splines}. 
Differential operators can also carry information about the data distribution. This is the case in semi-supervised learning, where the operator $D$ is learned from unlabeled data~\citep{zhu2003semi,zhou2005regularization,belkin2006manifold,slepcev2019analysis,cabannes2021overcoming}.
Finally, differential constraints can come from the physical knowledge of the object of study itself, as is often the case in scientific applications~\citep{cuomo2022scientific}. Typically taking the form of a PDE, such constraints provide complementary data, either on the same domain as the value measurements~\citep{sangalli2021spatial} or on a different domain. 
A classical example of the latter is the one where $Du^*$ is known inside a domain $\Omega$ while $u^*$ is known only on the boundary $\partial \Omega$, which is well studied in the PDE literature~\citep{evans2010partial}.

\textbf{Methods.~~}
Deep learning-based algorithms to tackle problems with differential constraints include the Deep Ritz Method~\citep{yu2018deep}, neural operators~\citep{lu2021learning,li2021fourier}, Universal Differential Equations~\citep{rackauckas2020universal}. Closest to our approach are Physics-Informed Neural Networks (PINNs)~\citep{raissi2019physics}, which minimize a loss function similar to ours. 
Kernel methods have also been considered to tackle the problems described above, from splines~\citep{kimeldorf1971some} to RBF collocation~\citep{kansa1990multiquadrics,zongmin1992hermite,fasshauer1996solving,wendland2004scattered}, semi-supervised learning \cite{belkin2006manifold} and, more recently, in scientific machine learning applications~\citep{owhadi2015bayesian,raissi2017gps,chen2021solving,doumeche2024physics,bona2026piks}. Our estimator belongs to this kernel-based family: it combines the usual kernel ridge loss on values with an empirical loss on differential observations \(Du(z_j)\).
An orthogonal class of methods that is commonly used to estimate PDE solutions are finite elements/volumes~\citep{azzimonti2015blood,sangalli2021spatial}.

\textbf{Theory.~~}
Many convergence results exist for learning PDE solutions with kernels~\citep{franke1998solving,wendland2004scattered,schaback2006kernel,batlle2025error} or neural networks~\citep{shin2020convergence,mishra2023estimates,shin2023error,de2024error}. In the settings considered, data observations are only available on the boundary of the input domain (and not inside the domain). 
Such settings differ from ours and are closer to classical numerical methods for PDEs such as finite differences or finite element methods.
Instead, we consider settings in which the value and differential observations lie in the same domain, and discuss the existing results for such settings, which are commonly of two types.
A first type is concerned with physical consistency: showing that a physics-informed estimator $\hat{u}$ converges to the target $u^*$ in a physically consistent, stronger than $L^2$ sense~\citep{doumeche25convergence,shi2010hermite,abdeen23sobolev}. 
In \cref{sec:piks_stronger_norm}, we provide  a result of this nature with learning rates in a norm which controls both $\hat{u}$ and $D\hat u$. In particular, \citet{doumeche25convergence} shows, for a broad class of differential operators that PINNs are physically consistent. 
For kernel estimators, \citet{shi2010hermite} and \citet{abdeen23sobolev} show that learning with gradients allows to obtain convergence rates in the $H^1$ norm. 
A second type of result focuses on how the physics-informed penalty benefits the $L^2$ convergence rates, which we show in \cref{sec:learning_rate_acceleration}.
Previously, \citet{fisher2025precise} studied the impact of gradient information on $L^2$ performance using a random feature model, and highlighted
the existence of a regime where incorporating gradient information may hurt the prediction accuracy. 
\citet{arnone2022some} investigated a regression estimator regularized with a 2-dimensional elliptic PDE, when $u^*$ belongs to the Sobolev space $H^2$.
Their estimator converges at least at the standard rate for $H^2$ functions ($n^{-2/3}$), and can reach the faster $n^{-4/5}$ rate when the PDE constraint is satisfied by the target. 
\citet{doumeche2024physics} reformulated empirical risk minimization with linear differential constraints as kernel regression in a physics-informed RKHS, yielding theoretical evidence that physical constraints can improve convergence rates, characterized through the effective dimension of the new kernel. Such improvements are illustrated for the setting of $D=\partial_x$ with homogeneous constraints ($Du^* = 0$). 
Subsequent works \citep{doumeche2025physics, doumeche2025fast} focus on approximations and fast implementations, while preserving the underlying convergence rates. 
Overall, these theoretical results either study an ideal estimator for which continuous constraints are enforced (which typically is not computable in closed form) \cite{arnone2022some,doumeche2024physics}, or tie the number of value and differential data points ($m=n$) \cite{shi2010hermite,abdeen23sobolev,fisher2025precise}.
In contrast, we provide results for any linear $D$ which hold for arbitrary finite values of $m$ and $n$, explicitly dealing with approximately enforced constraints.

Finally, we must mention works which characterized the statistical performance of kernel regression~\citep{caponnetto2007optimal, steinwart2009optimal, fischer2020sobolev}, which we extend to handle physical constraints. They introduce the key assumptions (source and capacity conditions) under which rates can be derived and the notion of \emph{effective dimension} which controls the difficulty of the learning problem.

\section{Physics-informed kernel regression}\label{sec:piks_setting}

\subsection{Problem setup}

We now give a formal description of the learning problem. Let $\mathcal X$ be a measurable space and let $u^*: \mathcal{X} \to \R$ denote the target function. We will consider linear differential operators $\dop$ of order $s\geq 1$, which take the form 
\begin{equation*}
    \dop = \sum_{|\alpha|\le s} c_\alpha \partial^\alpha,
\end{equation*}
with coefficients $c_\alpha\in\R$. We will assume (in~\ref{as:k_Cs}) that $\dop u^*$ is well defined. Let $\rho$ and $\rho_\dop$ be distributions over $\mathcal{X}$. We consider \textbf{value} and \textbf{differential} observations
\begin{alignat*}{3}
&y_i = u^*(x_i) + \varepsilon_i, \quad &&x_i \overset{i.i.d.}{\sim} \rho, \quad &&i=1,\dots,n,
\\
&d_j = (Du^*)(z_j) + \xi_j, \quad &&z_j \overset{i.i.d.}{\sim} \rho_D, \quad &&j=1,\dots,m,
\end{alignat*}
where \((\varepsilon_i)_i\) and \((\xi_j)_j\) are real-valued noise variables. In general, $\rho$ and $\rho_D$ can be different, reflecting the fact that value and differential information may be available on different regions of the domain.

\subsection{Physics informed kernel estimator}

Let $(\mathcal H,\langle\cdot,\cdot\rangle_{\mathcal H})$ be an RKHS with associated kernel $k: \mathcal X\times \mathcal X\to\R$.
Given parameters $\lambda > 0$ and $\gamma > 0$, the regularized empirical risk is defined, for $u\in\bmH$, as
\begin{equation}\label{eq:physics-informed-emp-risk}
\widehat R(u)
=
\frac1n\sum_{i=1}^n\big(u(x_i)-y_i\big)^2
+\gamma\,\frac1m\sum_{j=1}^m\big(Du(z_j)-d_j\big)^2
+\lambda\|u\|_{\mathcal H}^2,
\end{equation}
and the Physics Informed Kernel MethodS (PIKS) \cite{bona2026piks} estimator is the result of an optimization procedure in $\mathcal H$
\[
\hat u
:=
\operatorname*{argmin}_{u\in\mathcal H}\widehat R(u).
\]
A small deviation from the estimator defined in \cite{bona2026piks} is the introduction of the scaling parameter $\gamma$. The parameter \(\lambda\) controls the regularity of the estimator,
while \(\gamma\) controls the relative weight assigned to derivative observations. Setting \(\gamma=0\) recovers standard kernel ridge regression with value observations.

Since \(\lambda>0\), \(\widehat R\) is a strictly convex quadratic functional on \(\mathcal H\), hence the minimizer \(\hat u\) exists and is unique. Moreover, the estimator
admits a finite-dimensional representer form. In particular,
\[
\hat u(\cdot)
=
\sum_{i=1}^n \alpha_i k(x_i,\cdot)
+
\sum_{j=1}^m \beta_j D_1 k(z_j,\cdot),
\]
where the coefficients \(\alpha\in\mathbb R^n\) and \(\beta\in\mathbb R^m\) are the solutions to a linear system involving the value, mixed, and differential Gram matrices. The explicit system is given in Appendix~\ref{sec:closed_forms}.

\section{Main results}\label{sec:piks_rates}

In this section, we present learning rates for PIKS and quantify the impact of derivative information. After listing the assumptions in \cref{sec:assumptions}, we state a general finite sample result in \cref{sec:finite_sample_bounds}, from which, in \cref{sec:learning_rate_acceleration}, we derive learning rates, emphasizing the existence of two regimes. In \cref{sec:piks_stronger_norm} we study the rates in a  stronger norm involving both $u$ and $Du$.

\subsection{Assumptions}\label{sec:assumptions}

\Cref{as:D_order,as:k_Cs,as:bounded,as:attainability,as:consistency} are standard and ensure that derivative evaluations are well defined in the RKHS and that both value and derivative observations are generated by the same underlying target function. The main novel assumption is~\ref{ass:abstract-alignment}, which quantifies how much of the function space is visible through derivative information.

\begin{assumption}[Linear differential operator] \label{as:D_order} 
We consider linear differential operators $D$ of the form $D=\sum_{\lvert\alpha\rvert \leq s} c_\alpha \partial^{\alpha}$ for some integer $s \geq 1$ and coefficients $c_\alpha\in\mathbb{R}$.
\end{assumption}

Typical examples include directional derivatives, Laplacians, and more general linear partial differential operators. The order $s$ determines the degree of smoothness required in the kernel such that derivative evaluations are well defined.

\begin{assumption}[Kernel regularity]\label{as:k_Cs}
We assume that the kernel \(k\) is of class \(C^{s,s}\) on \(\bmX\times \bmX\), meaning that all mixed derivatives
\(\partial_x^\alpha \partial_{x'}^\beta k(x,x')\) with \(|\alpha|,|\beta|\le s\) exist and are continuous.
\end{assumption}

Under \cref{as:D_order,as:k_Cs}, the differential operator \(D\) can be applied to the kernel and gives rise to well-defined derivative representers in the RKHS. More precisely, let \(\phi(x):=k(x,\cdot)\) be the RKHS feature map. Since derivative evaluations are continuous linear functionals on $\mathcal H$~\citep[Corollary~4.36]{steinwart2008support}, and defining \(D\phi(x):=Dk(x,\cdot),\)
where $D$ acts on the first argument of $k$, one has $D\phi(x)\in\mathcal H$ and the reproducing property holds on both the kernel and its derivatives:
\[
u(x)=\langle u,\phi(x)\rangle_{\mathcal H},
\qquad
Du(z)=\langle u,D\phi(z)\rangle_{\mathcal H}.
\]

\begin{assumption}[Boundedness of features and outputs] \label{as:bounded}
There exist constants $\kappa_V,\kappa_D,L_V,L_D>0$ such that, almost surely,
\[
\|\phi(x)\|_{\mathcal H}\le \kappa_V,
\qquad
\|D\phi(z)\|_{\mathcal H}\le \kappa_D,
\qquad
|y|\le L_V,
\qquad
|d|\le L_D.
\]
\end{assumption}
Such boundedness assumptions are standard in statistical analyses of kernel methods and are satisfied under \cref{as:D_order,as:k_Cs} when \(\bmX\) is compact.

\begin{assumption}[Attainability] \label{as:attainability}
We assume that the regression function $u^*(x) := \mathbb E[y\mid x]$ belongs to $\mathcal H$.
Equivalently, there exists $w^*\in\mathcal H$ such that
\[
u^*(x) = \langle w^*,\phi(x)\rangle_{\mathcal H}.
\]
\end{assumption}
This is a standard assumption for well-specified problems: it states that the target function belongs to the RKHS induced by the kernel, so it can be represented by the model class under consideration.

\begin{assumption}[Consistency of derivative observations] \label{as:consistency}
We assume that the derivative observations are unbiased measurements of the derivative of the regression function, in the sense that
\[
\mathbb E[d\mid z] = Du^*(z).
\]
\end{assumption}
This assumption states that, when \(Du^*\) is known, it can be interpreted as having a physical constraint which is well-specified. It is also referred to as using a \emph{perfect} measure $\rho_D$~\citep{shi2010hermite}.

To quantify both the effective dimension of the hypothesis space and its alignment with value and derivative observations, we introduce the following covariance operators
\[
\Sigma := \mathbb E[\phi(x)\otimes\phi(x)],
\qquad
\Sigma_D := \mathbb E[D\phi(z)\otimes D\phi(z)],
\]
where, for \(a,b\in\mathcal H\), the rank-one operator \(a\otimes b : \mathcal H\to\mathcal H\) is defined as
\(
(a\otimes b)w := \langle w,b\rangle_{\mathcal H}\,a
\). For self-adjoint operators \(A,B\), we write \(A \preceq B\) when \(B-A\) is positive semi-definite.

\begin{assumption}[Value-derivative capacity decomposition] \label{ass:abstract-alignment}
Assume that there exist positive semi-definite operators \(\Sigma_1,\Sigma_2 \succeq 0\) such that
\[
\Sigma = \Sigma_1 + \Sigma_2,
\]
and that there exist constants $c_1,c_2>0$ and $\alpha,r\in[0,1]$ such that, for all
$\lambda\in(0,1]$,
\begin{equation}\label{eq:as6}
\Tr\!\big(\Sigma_1(\Sigma_1+\lambda I)^{-1}\big)
\;\le\;
c_1\,\lambda^{-\alpha},
\qquad
\Tr\!\big((\Sigma_D+\lambda I)^{-1}\Sigma_2\big)
\;\le\;
c_2\,\lambda^{-r}.
\end{equation}
\end{assumption}

This assumption decomposes the geometry of function values into directions that are \emph{invisible}, or only weakly visible, to the operator $D$ (captured by $\Sigma_1$) and directions that are \emph{detectable} through differential observations (captured by $\Sigma_2$). The exponent $\alpha$ governs the effective dimension of the invisible component, while $r$ measures how well the detectable component is separated from the null-space of $D$. In the limiting case \(\Sigma_1=\Sigma\) and \(\Sigma_2=0\), the assumption reduces to the usual effective-dimension condition of standard kernel regression. 
Together, these conditions formalize how differential information can reduce the statistical complexity of learning function values via capacity reduction, depending on interactions between kernel, differential operator, and data distribution.

\subsection{Finite-sample bounds}\label{sec:finite_sample_bounds}

\begin{theorem}[High-probability finite-sample bounds for PIKS]\label{thm:finite-sample} Under \cref{as:D_order,as:k_Cs,as:bounded,as:attainability,as:consistency}, let $\lambda>0$, $\gamma>0$ and let $\delta\in(0,1/2)$. Define
\[
\Sigma_\gamma := \Sigma + \gamma \Sigma_D,
\qquad
\mathcal{B}(\lambda,\gamma):=\big\|(\Sigma_\gamma + \lambda I)^{-1/2}w^*\big\|_{\bmH},
\]
and the (value and derivative) capacity terms
\[
\mathcal{N}_V(\lambda,\gamma)
:=
\big\|(\Sigma_\gamma + \lambda I)^{-1/2}\Sigma^{1/2}\big\|_{\mathrm{HS}},
\qquad
\mathcal{N}_D(\lambda,\gamma)
:=
\big\|(\Sigma_\gamma + \lambda I)^{-1/2}\Sigma_D^{1/2}\big\|_{\mathrm{HS}}.
\]
If
\begin{equation}\label{eq:lambda-cond}
\frac{36}{n}\log\frac{2n}{\delta} \;\le\; \lambda \;\le\; \|\Sigma\|_\infty \quad\text{and}\quad
\frac{36}{m}\log\frac{2m}{\delta} \;\le\; \lambda\gamma^{-1} \;\le\; \|\Sigma_D\|_\infty,
\end{equation}
then there exists a constant $c>0$, depending only on $\kappa_V,\kappa_D,L_V,L_D$ and $\|u^*\|_{\bmH}$, but not on $\lambda,\gamma,n,m,\delta$, such that, with probability at least $1-2\delta$,
\begin{equation}\label{eq:main-bound}
\E\left[\left(\hat u(x)-u^*(x)\right)^2\right]^{\frac{1}{2}}
\le
c\log\frac{4}{\delta}\left[
\frac{1}{\sqrt{\lambda}}
\left(\frac1n+\frac{\gamma}{m}\right)
+
\frac{\mathcal N_V(\lambda,\gamma)}{\sqrt n}
+
\gamma\frac{\mathcal N_D(\lambda,\gamma)}{\sqrt m}
+
\lambda\mathcal B(\lambda,\gamma)
\right].
\end{equation}
\end{theorem}
\begin{sproof}
The proof stems from decomposing the error into separate approximation and estimation terms, and controlling the stochastic parts through concentration inequalities for covariance operators in Hilbert spaces.
The main novelty compared to the classical kernel ridge regression analysis lies in handling of the derivative observations. We introduce the covariance operator
$\Sigma_\gamma = \Sigma + \gamma \Sigma_D$,
and decompose all deviation terms into value and derivative contributions while carefully tracking their dependence on \(\gamma\), \(n\), and \(m\). 
Once these additional decompositions are established, the remainder of the argument closely parallels the standard KRR proof~\citep{caponnetto2007optimal}.
\end{sproof}
The bound in \cref{thm:finite-sample} separates the contributions of value observations, derivative observations, and regularization bias. The quantities $\mathcal N_V(\lambda,\gamma)$ and $\mathcal N_D(\lambda,\gamma)$ play the role of effective dimensions associated with the value and derivative components of the problem. They control the corresponding variance contributions, while the bias term $\lambda \mathcal B(\lambda,\gamma)$ captures the approximation error induced by regularization. The contributions of $\mathcal N_V(\lambda,\gamma)$ and $\mathcal N_D(\lambda,\gamma)$ interact through the shared parameters $\lambda$ and $\gamma$. Increasing $\gamma$ places more weight on the derivative constraints, which reduces the effective complexity of the value component by shrinking the set of admissible functions in directions that are observable through $D$, but also amplifies the variance contribution coming from noisy derivative observations. Setting $\gamma = 0$ recovers the standard KRR bound~\citep{caponnetto2007optimal}. Setting $m=\infty$ instead recovers the rates of \citet{doumeche2024physics}.

\subsection{Learning rate acceleration from differential information}\label{sec:learning_rate_acceleration}

We now instantiate the finite-sample bound under \cref{ass:abstract-alignment} and optimize over the regularization parameters $\lambda$ and $\gamma$ to precisely characterize the learning rates of the PIKS estimator.

\begin{figure}
    \centering
    \includegraphics[height=5.7cm]{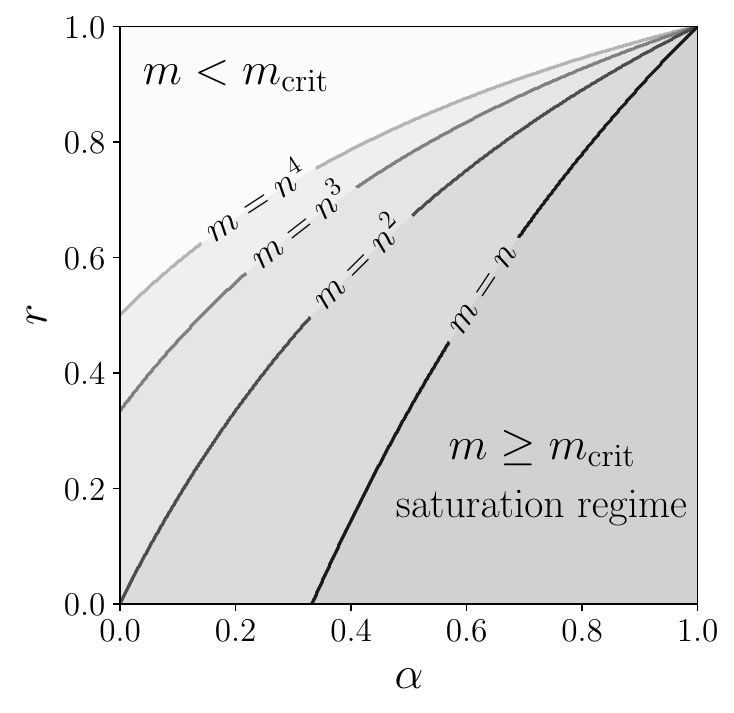} \hfill
    \includegraphics[height=5.7cm]{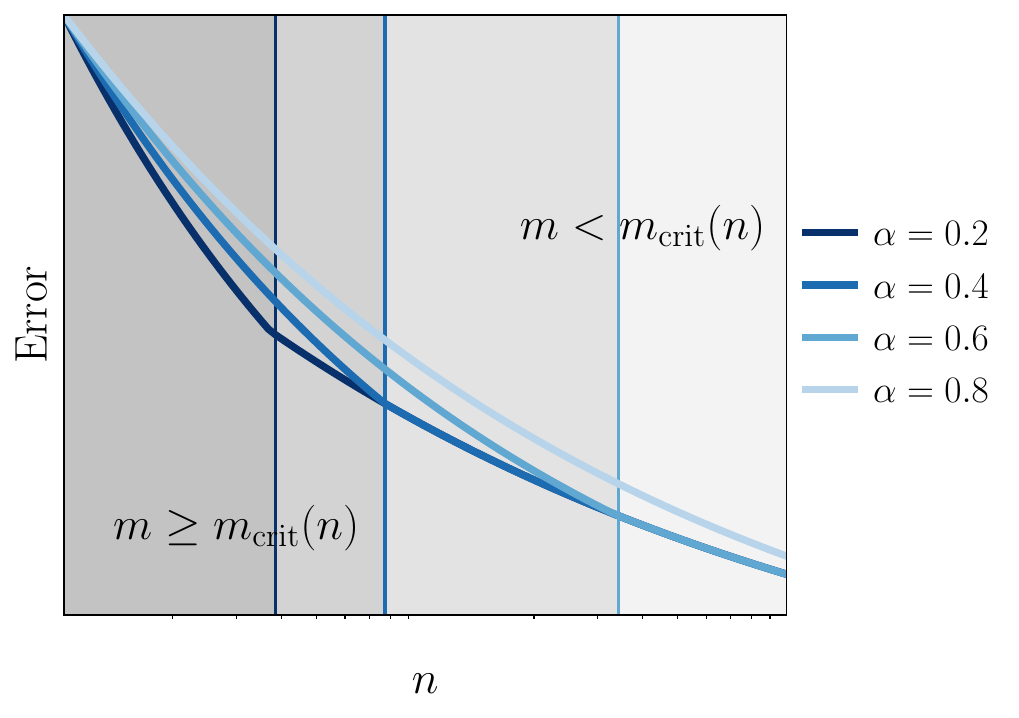}
    \caption{(\emph{left}) Rate-regime changes in the $(\alpha, r)$ plane. On the lower-right of each contour line the problem parameters are such that $m$ is larger than the critical threshold. (\emph{right}) Log-linear plot of the expected error as a function of $n$ and $\alpha$ for fixed $m, r$. For smaller $\alpha$ a larger $m/n$ ratio is required to be in the fast saturated regime.}
    \label{fig:abs-rates}
\end{figure}

\begin{corollary}[Learning rates for PIKS]
\label{thm:rates} Under \cref{as:D_order,as:k_Cs,as:bounded,as:attainability,as:consistency,ass:abstract-alignment}, there exist choices of regularization parameters $\lambda=\lambda(n,m)$ and $\gamma=\gamma(n,m)$, given explicitly in the proof and which depend polynomially on $n$, $m$, $\log n$, and $\log m$, such that, with high probability,
\[
\E\left[\left(\hat u(x) - u^*(x)\right)^2\right]^{1/2}
\;\lesssim\;
\begin{cases}
n^{-\frac14}m^{-\frac{1-r}{8}},
& m \lesssim m_{\mathrm{crit}},\\[8pt]
n^{-\frac{1}{2(1+\alpha)}},
& m \gtrsim m_{\mathrm{crit}},
\end{cases}
\]
where
\(
m_{\mathrm{crit}}
\;\asymp\;
n^{\frac{2(1-\alpha)}{(1+\alpha)(1-r)}}.
\)
\end{corollary}
Here and throughout, $a(n,m) \lesssim b(n,m)$ means that $a(n,m) \le c\,b(n,m)$ for some $c>0$
independent of $n$ and $m$, up to logarithmic factors in $n$, $m$, and $\delta^{-1}$ (the high-probability parameter).

The error bound reveals two regimes depending on the amount of differential samples $m$, relative to the number of function-value samples $n$ as depicted also in \cref{fig:abs-rates}:

\begin{itemize}

\item \emph{Differential-limited regime ($m \lesssim m_{\mathrm{crit}}$).} 
With low $m$, the error decreases as both $n$ and $m$ increase. In particular as the number of differential observations $m$ increases, the estimation error along directions captured by $\Sigma_2$ decreases. 
The rate of decrease depends on how well the $\Sigma_2$ directions are aligned with the differential covariance $\Sigma_D$. The exponent $r$ quantifies this effect, with smaller values of $r$ leading to faster decay of the error with respect to $m$.

\item \emph{Saturation threshold ($m \asymp m_{\mathrm{crit}}$).} Increasing $m$, a threshold is reached where differential and value contributions to the error are of the same order:
\[
n^{-\frac14} m^{-\frac{1-r}{8}}
\;\asymp\;
n^{-\frac{1}{2(1+\alpha)}},
\]
which yields 
\(
m_{\mathrm{crit}} \asymp n^{\frac{2(1-\alpha)}{(1+\alpha)(1-r)}}.
\)
The dependence of $m_{\mathrm{crit}}$ on $\alpha$ and $r$ follows from this relation: smaller values of $r$ lead to a slower decay of the differential term in $m$, and smaller values of $\alpha$ lead to a faster decay of the value term in $n$, so that a larger $m$ is required to balance the two contributions.

\item \emph{Saturation regime ($m \gtrsim m_{\mathrm{crit}}$).}
Once \(m\) exceeds the threshold, the rate saturates: increasing the number of differential observations only reduces lower-order terms, but it no longer improves the leading error term.  At this point, the error on the differential-accessible component is below the error driven by \(\Sigma_1\), which captures directions that are either in the null-space of \(D\), and therefore invisible to differential observations, or only weakly visible through them. These directions can be estimated only with value samples.
In this regime, the leading rate matches the physical oracle rate obtained when the differential information is known exactly (see Appendix~\ref{sec:oracle_physical_rates}).
\end{itemize}

\subsection{Learning rates in physically consistent norms}\label{sec:piks_stronger_norm}

In many physics-informed settings, respecting the differential structure is as important as minimizing prediction error. For instance, in learning interatomic potentials, predicting differential quantities such as atomic forces is essential for molecular dynamics simulations~\citep{noe20machine}. However, improved learning rates for function values do not by themselves ensure that \(\hat{u}\) captures this structure. Indeed, convergence in value alone does not control derivatives: on any bounded domain, \(u_\varepsilon(x)=\varepsilon\sin(\varepsilon^{-2}x)\) satisfies \(\|u_\varepsilon\|_{L^2}\to 0\) as \(\varepsilon\to 0\), whereas for \(D=\partial_x\), \(Du_\varepsilon(x)=\varepsilon^{-1}\cos(\varepsilon^{-2}x)\) and \(\|Du_\varepsilon\|_{L^2}\to\infty\). We therefore study convergence in a stronger \(L^2\)-type norm that jointly controls function-value and differential errors.

\begin{corollary}[Learning rates in the physically consistent norm]
Under \cref{as:D_order,as:k_Cs,as:bounded,as:attainability,as:consistency,ass:abstract-alignment} there exist choices of regularization parameters $\lambda=\lambda(n,m)$ and $\gamma=\gamma(n,m)$, given explicitly in the proof with $\gamma\ge 1$ such that, with high probability,
\[
\E\left[
    \left(\hat u(x) - u^*(x)\right)^2 + \left(D\hat u(z) - Du^*(z)\right)^2
\right]^{1/2}
\;\lesssim\;
\begin{cases}
n^{-\frac14}m^{-\frac{1-r}{8}} + m^{-1/4},
&
m \lesssim m_{\mathrm{crit}}
\\[6pt]
n^{-\frac{1}{2(1+\alpha)}} + m^{-1/4},
&
m \gtrsim m_{\mathrm{crit}},
\end{cases}
\]
where
\(
m_{\mathrm{crit}}
\;\asymp\;
n^{\frac{2(1-\alpha)}{(1+\alpha)(1-r)}}.
\)
\end{corollary}
The stronger norm bounds rule out the possibility that the PIKS estimator has a good prediction 
accuracy on $u^*$ while failing to accurately learn $Du^*$.
This can be interpreted as a form of physical consistency: the learned function respects both the values and the differential structure of the target. Compared with the value-only rates in \cref{thm:rates}, this stronger guarantee incurs an additional term $m^{-1/4}$, which depends only on the number of derivative observations and therefore cannot be reduced by increasing $n$ alone. This is a reasonable price to pay in settings where differential predictions are themselves of interest or are used in downstream tasks.

\section{Examples}\label{sec:examples}

Here we illustrate the improved rates of \cref{thm:rates} in concrete learning settings. We focus in \cref{ex:sobolev-torus} on Sobolev spaces on the torus where Laplacian information can provide large benefits, and in \cref{sec:grad-ex} on bounded domains, which are particularly relevant for their connection to PDEs.

\subsection{Example 1: Sobolev spaces on the torus}\label{ex:sobolev-torus}

We define the periodic Mat\'ern kernel on $\mathbb T^d= [0,1]^d$ with Fourier expansion
\[
k(x-x')=\sum_{k\in\mathbb Z^d}\mu_k\,e^{2\pi i k\cdot(x-x')},
\qquad
\mu_k \asymp (1+|k|^2)^{-s},
\]
for some $s>\frac{d}{2}+2$, where $\asymp$ denotes the equality up to positive multiplicative constants. The associated RKHS $\mathcal H$ consists of functions
$f(x)=\sum_{k\in\mathbb Z^d}c_k\,e^{2\pi i k\cdot x}$  with norm
\[
\|f\|_{\mathcal H}^2
=\sum_{k\in\mathbb Z^d}\frac{|c_k|^2}{\mu_k}
\asymp
\sum_{k\in\mathbb Z^d}(1+|k|^2)^s|c_k|^2,
\] where $c_k\in\mathbb C$ are the Fourier coefficients. $\mathcal H$ is norm-equivalent to the Sobolev space $H^s(\mathbb T^d)$.

\begin{proposition}[Capacity decomposition for Sobolev spaces]
\label{lem:sobolev_capacity}
Let $\mathcal X=\mathbb T^d $, let $k$ be a periodic Matérn kernel of smoothness $s>\frac{d}{2}+2$, and let $D=\sum_{i\in S}\partial_{z_i}^2$ be a partial Laplacian with $S\subseteq\{1,\dots,d\}$.
Assume that value and differential samples are drawn uniformly on $\mathbb T^d$.
Then the value covariance $\Sigma$ admits a decomposition $\Sigma=\Sigma_1+\Sigma_2$ satisfying \cref{ass:abstract-alignment}, with exponents
\[
\alpha=\frac{d-|S|}{2s},
\qquad
r=\frac{d}{2s}.
\]
\end{proposition}

Referring back to \cref{thm:rates}, this Sobolev example provides a functional-analytic interpretation of \cref{ass:abstract-alignment} and of the learning rates.
For small $m$, the learning rate coincides with that of estimating a function in $H^s(\mathbb T^d)$ from value observations, reflecting the full $d$-dimensional complexity of the data. 
When $m$ exceeds the saturation threshold, all $n$ function-value data points can be used to learn the nullspace of the partial Laplacian $D$, which depends only on $d_0 = d-\lvert S \rvert$ variables. The saturated rate therefore matches the minimax rate for Sobolev regression on $\mathbb T^{d_0}$, making explicit how the differential constraint reduces the problem dimensionality.

\subsection{Example 2: Gradients on bounded domain}\label{sec:grad-ex}

Consider now the case when $D$ is the gradient operator $\nabla$, and the RKHS $\bmH$ is a Sobolev space of smoothness $s>\frac{d}{2}+1$ on domain $\mathcal X\subset\R^d$. We consider the reproducing kernel $k$ associated with $\bmH$ to be a Mat\'ern kernel of smoothness $s$~\citep{williams2006gaussian}.

Note that, while in \cref{sec:piks_setting,sec:piks_rates} the analysis is limited to the case of scalar-valued $D$, it still holds for vector-valued operators. Denoting by $D_i$ each scalar component of $D$, we decompose the differential-data covariance as $\Sigma_D :=\sum_{\ell=1}^k\Sigma_{D_\ell}$, where $\Sigma_{D_\ell} := \EE [D_\ell \phi(z) \otimes D_\ell \phi(z)]$. Provided this new $\Sigma_D$ satisfies \cref{ass:abstract-alignment} (which is unchanged), \cref{thm:rates} still holds. We give more details on the extension to the vector-valued case in Appendix~\ref{sec:adaptation-vector}. 

When $D=\nabla$, we have
\(
    \Sigma_D = \EE\left[\sum_{i=1}^d \partial_i \phi(z) \otimes \partial_i \phi(z)\right].
\) The effect of learning with differential constraints is shown in the following proposition.

\begin{proposition}\label{prop:gradient-rates} Let $\mathcal X\subset\R^d$ be a bounded, connected and Lipschitz domain.
Fix $s>\frac d2 +1$. Let $\bmH = H^s(\Omega)$ and let $D = \nabla$ be the gradient operator. Then the value covariance operator $\Sigma$ admits a decomposition $\Sigma = \Sigma_1+\Sigma_2$ satisfying \cref{ass:abstract-alignment}, with exponents
\[
\alpha = 0,
\qquad
r = \frac{d}{2(s-1)}.
\]
\end{proposition}

The \emph{saturated} rate from \cref{thm:rates} is obtained as long as $m \geq m_{crit} = n^{\frac{2(s-1)}{(s-1)-d/2}} \geq n^2$. Then,
\[
    \E\left[\left(\hat u(x) - u^*(x)\right)^2\right]^{1/2} \lesssim n^{-\frac{1}{2(1+\alpha)}},
\]
which proves that the gradient information allows to obtain the parametric rate $n^{-\frac{1}{2}}$.
\section{Numerical experiments}\label{sec:experiments}

We show how the different rate-regimes of \cref{thm:rates} look like in practice through simulations on two learning problems.
We will show i) the \emph{error saturation effect} when $m$ increases and ii) the effect of decreasing the capacity of $\Sigma_1$ (decreasing $\alpha$) via the partial Laplacian. More details on the implementation are available in Appendix~\ref{sec:app-experiments}.

\textbf{Saturation effect on Mat{\'e}rn data.}~~
We sample synthetic data on the 2d disk according to the following function, letting $x_\mathrm{supp}$ be a support point in the domain and $p := \| x - x_\mathrm{supp}\|$
\begin{equation*}
    u^*(x) = \left(1 - \|x\|^2\right) \left[  \left(1 + \sqrt{5} p + (5/3) p^2\right)\exp\left(-\sqrt{5}p\right) \right].
\end{equation*}
$u^*$ is in $H^s$ with $s=3$, and we will use the corresponding Mat{\'e}rn ($\nu=5/2$) kernel.
In \cref{fig:inc-m} we show that, for fixed $n$ and increasing $m$, the error rates switch between the two regimes of \cref{thm:rates}: an initial phase in which $m < m_{\mathrm{crit}}$ and the error decreases as $m$ grows, followed by a phase in which $m$ has grown beyond the threshold and the error has reached a \emph{saturation} with respect to $m$ in which additional data does not improve accuracy.
Here we used $D=\nabla$ and plotted empirical rates using PIKS next to the theoretical results from \cref{sec:grad-ex}. Despite some differences -- notably the experimental curves saturate at similar values of $m$ -- due to the unknown problem-dependent constants affecting the rates, the saturation effect is clear as well as the dependence of the initial rate on both $n$ and $m$.

\begin{figure}
    \centering
    \begin{minipage}[t]{0.48\linewidth}
        \includegraphics[width=\linewidth]{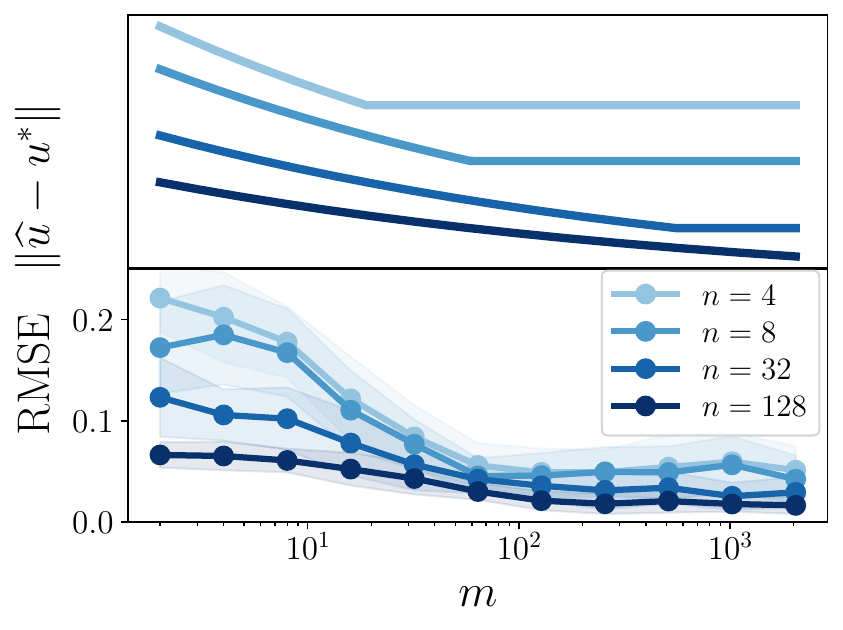}\vspace*{-0.5em}
        \caption{Theoretical and experimental rates for fixed $n$, increasing $m$ showing the saturation effect with gradient information on Mat\'ern data.}\label{fig:inc-m}
    \end{minipage}
    \hfill
    \begin{minipage}[t]{0.48\linewidth}
        \includegraphics[width=\linewidth]{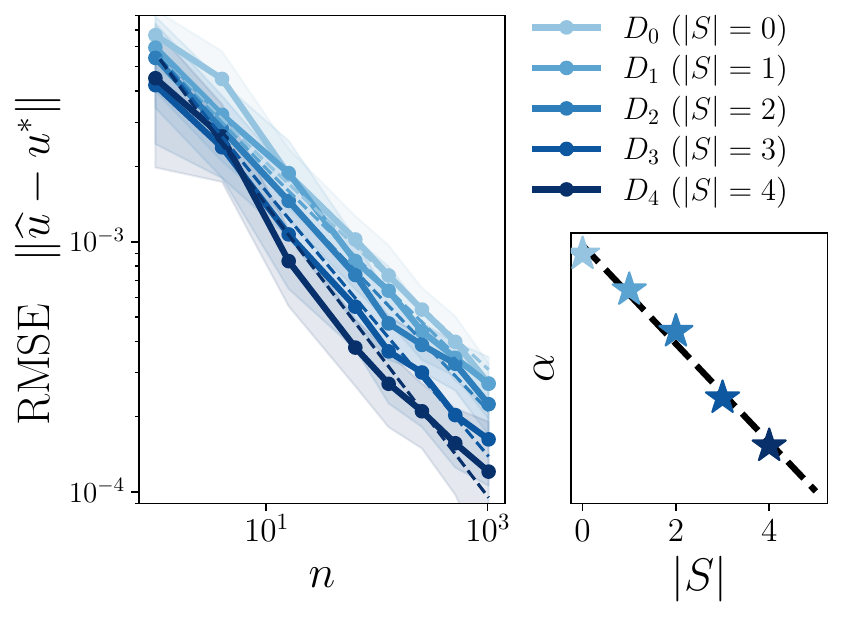}\vspace*{-0.5em}
        \caption{Test error rates increasing $n$ and partial Laplacian dimensions $\lvert S \rvert$. Dashed lines show a best fit for the rates of \cref{lem:sobolev_capacity}. On the right are the inferred $\alpha$ exponents.}\label{fig:part-lap}
    \end{minipage}
\end{figure}

\textbf{Partial Laplacian.}~~We now consider Sobolev spaces on the torus with partial Laplacian information. Unlike the example in \cref{ex:sobolev-torus}, we approximate the infinite-dimensional Fourier construction by truncating to a finite number of frequencies.
With $\cX =\mathbb T^4$, we randomly sample frequency vectors $k_\ell\in\mathbb{Z}^4$ such that a single dimension is active (non-zero) at a time.
Writing $\phi_\ell(x) = \sqrt{2}\cos(2\pi\langle k_\ell, x\rangle)$ for the feature-map, target function and kernel are defined as
\begin{equation*}
    u^*(x) = u_0 + \sum_{\ell=1}^Fc_\ell \phi_\ell(x), \qquad k(x, x') = 1 + \sum_{\ell=1}^F\mu_\ell\phi_\ell(x)\phi_\ell(x'),
\end{equation*}
with $F$ frequency coefficients $c_\ell$ and $\mu_\ell$ which depend on a smoothness parameter of the problem.
$k_\ell$ are sampled such that $u^*$ decomposes into frequency blocks $u^*(x) = u_0 + u_1(x) + u_2(x) + u_3(x) + u_4(x)$, 
each of which only depends on a single dimension of $x$ and cannot be observed unless the partial Laplacian $D_s = \sum_{i=1}^s\partial^2_{z_i}$ includes that specific dimension. This creates a simple setting in which the incrementing the partial Laplacian order $s$, increases the amount of information which can be learned with the differential data. Moreover the Sobolev capacity decomposition from \cref{lem:sobolev_capacity} remains valid after truncation, with constants independent of $F$. 

In \cref{fig:part-lap} we observe the learning rate in $n$, as the type of differential data changes: from having access to function data only ($D_0$) to having access to the full Laplacian ($D_4$).
The number of differential points $m$ is set as a constant factor of $n$. 
Best fits are obtained for the theoretical rates of \cref{thm:finite-sample} to obtain the exponent term of the error rate in $n$.
Despite the high variance of the experiments, due to the noise $\varepsilon, \xi$ and the random sampling of points $x_i, z_j$, it is evident that with each increase in the amount of differential information $\lvert S \rvert$ the learning rate increases.
From the fitted curves we can infer the values of $\alpha$ which -- as predicted by \cref{lem:sobolev_capacity} -- have a linear relationship with the number of partial Laplacian dimensions (denoted by $\lvert S \rvert$), despite the mismatch caused by the finite-dimensional setting.

\section{Conclusion and research directions}

In this paper, we provide a precise characterization of how kernel regression benefits from learning with both function-value and differential observations. Our analysis quantifies how the prediction error depends jointly on the number of value samples, the number of differential samples, and the structure of the underlying differential operator. The resulting rates reveal two regimes: a differential-limited regime, in which the error decreases with both types of samples, and a saturation regime, in which the leading rate matches the physical oracle rate attainable with exact differential information. We also establish convergence in a stronger norm that controls both function-value and differential errors. Several questions remain open: are these rates minimax optimal? Can similar gains be obtained for nonlinear differential operators or misspecified physical constraints? 
Is it possible to preserve the statistical gains with approximate algorithms which are more efficient?

\section*{Acknowledgements}
This material is based upon work supported by the European Commission (Horizon Europe grant ELIAS 101120237), and the Ministry of Education, University and Research (FARE grant ML4IP R205T7J2KP).

{\small
\bibliographystyle{unsrtnat}
\bibliography{references}
}

\newpage
\appendix

\section{Proofs}\label{sec:piks_proofs}

This section provides the proofs of the main results. We first derive the closed-form expressions of the empirical and population minimizers in Section~\ref{sec:closed_forms}, then establish an error decomposition in Section~\ref{sec:error_decomposition}, derive high-probability bounds for its constituent terms in Section~\ref{sec:hp-bounds}, and finally combine these results to obtain finite-sample guarantees and optimized learning rates in Sections~\ref{sec:finite_sample} and~\ref{sec:learning_rates}.

\subsection{Closed forms}\label{sec:closed_forms}

We derive here the closed-form expressions of the empirical minimizer of the regularized empirical risk $\widehat R$ introduced in Section~\ref{sec:piks_setting},
\[
\widehat R(u)
=
\frac1n\sum_{i=1}^n\big(u(x_i)-y_i\big)^2
+\gamma\,\frac1m\sum_{j=1}^m\big(Du(z_j)-d_j\big)^2
+\lambda\|u\|_{\mathcal H}^2,
\]
and of the minimizer of its population counterpart
\[
R_{\lambda,\gamma}(u)
=
\mathbb E\big[(u(x)-y)^2\big]
+
\gamma\,\mathbb E\big[(Du(z)-d)^2\big]
+
\lambda\|u\|_{\mathcal H}^2,
\]
where the expectations are taken over \(x\sim\rho\), \(z\sim\rho_D\), and the corresponding observation noises. The empirical minimizer is given in operator form in Lemma~\ref{lem:empirical-solution} and in finite-dimensional form in Lemma~\ref{lem:piks-finite-dimensional}. The population minimizer is given in Lemma~\ref{lem:population-solution}.

\vspace{1em}
\begin{lemma}[Empirical solution]\label{lem:empirical-solution}
Let \(\lambda>0\) and \(\gamma>0\). Define
\[
\widehat \Sigma
:= \frac{1}{n}\sum_{i=1}^n \phi(x_i)\otimes \phi(x_i),
\qquad
\widehat \Sigma_D
:= \frac{1}{m}\sum_{j=1}^m D\phi(z_j)\otimes D\phi(z_j),
\]
\[
\widehat V
:= \frac{1}{n}\sum_{i=1}^n y_i\,\phi(x_i),
\qquad
\widehat V_D
:= \frac{1}{m}\sum_{j=1}^m d_j\,D\phi(z_j).
\]
Then \(\widehat R\) has a unique minimizer \(\hat u \in\mathcal H\) of the form
\[
\hat u(x) = \langle \widehat w,\phi(x)\rangle_{\mathcal H}, \qquad \widehat w
=
\big(\widehat \Sigma + \gamma \widehat \Sigma_D + \lambda I\big)^{-1}
\big(\widehat V + \gamma \widehat V_D\big).
\]
\end{lemma}

\begin{proof}
Using the representation \(u(x)=\langle w,\phi(x)\rangle_{\mathcal H}\),
\(Du(z)=\langle w,D\phi(z)\rangle_{\mathcal H}\), we write \(\widehat R(u)=\widehat R(w)\) as
\[
\widehat R(w)
=
\frac1n\sum_{i=1}^n\big(\langle w,\phi(x_i)\rangle_{\bmH} - y_i\big)^2
+\gamma\,\frac1m\sum_{j=1}^m\big(\langle w,D\phi(z_j)\rangle_{\bmH} - d_j\big)^2
+\lambda\langle w,w\rangle_{\mathcal H}.
\]

Expanding the least-squares terms in \(\mathcal H\), we obtain for the value part
\[
\begin{aligned}
\frac1n\sum_{i=1}^n\big(\langle w,\phi(x_i)\rangle - y_i\big)^2
&= \frac1n\sum_{i=1}^n
\Big(\langle w,\phi(x_i)\rangle^2
- 2y_i \langle w,\phi(x_i)\rangle
+ y_i^2\Big)\\
&= \left\langle w,
\Big(\tfrac1n\sum_{i=1}^n \phi(x_i)\otimes \phi(x_i)\Big) w
\right\rangle_{\mathcal H}
- 2\left\langle \tfrac1n\sum_{i=1}^n y_i\phi(x_i),\,w\right\rangle_{\mathcal H}
+ \text{const}\\
&= \langle w,\widehat \Sigma w\rangle_{\mathcal H}
- 2\langle \widehat V,w\rangle_{\mathcal H}
+ \text{const},
\end{aligned}
\]
and for the derivative part
\[
\begin{aligned}
\gamma\,\frac1m\sum_{j=1}^m\big(\langle w,D\phi(z_j)\rangle - d_j\big)^2
&= \gamma\left(
\langle w,\widehat \Sigma_D w\rangle_{\mathcal H}
- 2\langle \widehat V_D,w\rangle_{\mathcal H}
+ \text{const}
\right).
\end{aligned}
\]

Adding the regularization term \(\lambda\langle w,w\rangle_{\mathcal H}\), we get
\[
\widehat R(w)
=
\big\langle w,(\widehat \Sigma + \gamma \widehat \Sigma_D + \lambda I)w\big\rangle_{\mathcal H}
- 2\big\langle \widehat V + \gamma \widehat V_D, w\big\rangle_{\mathcal H}
+ \text{const}.
\]
Since $\widehat\Sigma$ and $\widehat\Sigma_D$ are positive self-adjoint and $\lambda>0$, the operator $\widehat\Sigma+\gamma\widehat\Sigma_D+\lambda I$ is strictly positive. 
Thus \(\widehat R(w)\) is a strictly convex quadratic functional of \(w\), whose unique minimizer \(\widehat w\) is characterized by the normal equation
\[
(\widehat \Sigma + \gamma \widehat \Sigma_D + \lambda I)\,\widehat w
= \widehat V + \gamma \widehat V_D,
\]
which yields
\[
\widehat w
=
\big(\widehat \Sigma + \gamma \widehat \Sigma_D + \lambda I\big)^{-1}
\big(\widehat V + \gamma \widehat V_D\big).
\]
The corresponding minimizer in function space is
\(\hat u(x)=\langle \widehat w,\phi(x)\rangle_{\mathcal H}\).
\end{proof}

\vspace{1em}
\begin{lemma}[Finite-dimensional form of the PIKS estimator]
\label{lem:piks-finite-dimensional}
Let \(\lambda>0\) and \(\gamma>0\). Define the Gram matrices
\[
(K_{XX})_{ii'} = k(x_i,x_{i'}), 
\qquad
(K_{XZ})_{ij} = D_1 k(z_j,x_i),
\]
\[
(K_{ZX})_{ji} = D_2 k(x_i,z_j),
\qquad
(K_{ZZ}^{D})_{jj'} = D_1D_2 k(z_j,z_{j'}).
\]
Then the PIKS estimator can be computed as
\[
\hat u(\cdot)
=
\frac1{\sqrt n}\sum_{i=1}^n \alpha_i k(x_i,\cdot)
+
\sqrt{\frac{\gamma}{m}}\sum_{j=1}^m \beta_j D_1k(z_j,\cdot),
\qquad
\theta=
\begin{bmatrix}
\alpha\\
\beta
\end{bmatrix},
\]
where \(\theta\in\mathbb R^{n+m}\) solves
\[
\left(
\begin{bmatrix}
\tfrac1n K_{XX}
&
\sqrt{\tfrac{\gamma}{nm}}\,K_{XZ}
\\[4pt]
\sqrt{\tfrac{\gamma}{nm}}\,K_{ZX}
&
\tfrac{\gamma}{m}K_{ZZ}^{D}
\end{bmatrix}
+
\lambda I
\right)\theta
=
\begin{bmatrix}
y/\sqrt n\\
\sqrt{\gamma}\,d/\sqrt m
\end{bmatrix}.
\]
\end{lemma}

\begin{proof}
Let \((e_i)_{i=1}^n\) and \((f_j)_{j=1}^m\) denote the canonical bases of \(\mathbb R^n\) and \(\mathbb R^m\), and define
\[
\Phi e_i=k(x_i,\cdot),
\qquad
D\Phi f_j=D_1k(z_j,\cdot).
\]
Set
\[
\Gamma_\gamma=
\Big[
\tfrac1{\sqrt n}\Phi\quad \sqrt{\tfrac{\gamma}{m}}D\Phi
\Big],
\qquad
b_\gamma=
\begin{bmatrix}
y/\sqrt n\\
\sqrt{\gamma}\,d/\sqrt m
\end{bmatrix}.
\]
Then
\[
\widehat R(u)=\|\Gamma_\gamma^\ast u-b_\gamma\|_2^2+\lambda\|u\|_{\mathcal H}^2.
\]
The optimality condition gives
\[
(\Gamma_\gamma\Gamma_\gamma^\ast+\lambda I)\hat u=\Gamma_\gamma b_\gamma.
\]
Using
\[
(\Gamma_\gamma\Gamma_\gamma^\ast+\lambda I)^{-1}\Gamma_\gamma
=
\Gamma_\gamma(\Gamma_\gamma^\ast\Gamma_\gamma+\lambda I)^{-1},
\]
we get
\[
\hat u=\Gamma_\gamma\theta,
\qquad
\theta=(\Gamma_\gamma^\ast\Gamma_\gamma+\lambda I)^{-1}b_\gamma.
\]
Finally,
\[
\Gamma_\gamma^\ast\Gamma_\gamma
=
\begin{bmatrix}
\tfrac1n K_{XX}
&
\sqrt{\tfrac{\gamma}{nm}}K_{XZ}\\[4pt]
\sqrt{\tfrac{\gamma}{nm}}K_{ZX}
&
\tfrac{\gamma}{m}K_{ZZ}^{D}
\end{bmatrix},
\]
which gives the claimed system. Moreover, since
\[
\hat u=\Gamma_\gamma\theta
=
\frac1{\sqrt n}\Phi\alpha
+
\sqrt{\frac{\gamma}{m}}D\Phi\beta,
\]
we obtain
\[
\hat u(\cdot)
=
\frac1{\sqrt n}\sum_{i=1}^n \alpha_i k(x_i,\cdot)
+
\sqrt{\frac{\gamma}{m}}\sum_{j=1}^m \beta_jD_1k(z_j,\cdot).
\]
\end{proof}

\vspace{1em}
\begin{lemma}[Population solution]\label{lem:population-solution}
Let \(x\sim \rho\) and \(z\sim \rho_D\), and let the expectations involving \(y\) and \(d\) be taken with respect to the joint laws induced by the observation models. Define
\[
\Sigma := \mathbb E_{x\sim \rho}\big[\phi(x)\otimes \phi(x)\big],
\qquad
\Sigma_D := \mathbb E_{z\sim \rho_D}\big[D\phi(z)\otimes D\phi(z)\big],
\]
\[
V := \mathbb E\big[y\,\phi(x)\big],
\qquad
V_D := \mathbb E\big[d\,D\phi(z)\big].
\]
Then \(R_{\lambda,\gamma}\) has a unique minimizer \(u_{\lambda,\gamma}\in\mathcal H\) of the form
\[
u_{\lambda,\gamma}(x)=\langle w_{\lambda,\gamma},\phi(x)\rangle_{\mathcal H},
\qquad
w_{\lambda,\gamma}
=
\big(\Sigma + \gamma \Sigma_D + \lambda I\big)^{-1}(V + \gamma V_D).
\]
\end{lemma}

\begin{proof}
The derivations are the same as in the empirical case, (Lemma~\ref{lem:empirical-solution}) with empirical averages replaced by expectations and \(\widehat \Sigma,\widehat \Sigma_D,\widehat V,\widehat V_D\) replaced by \(\Sigma,\Sigma_D,V,V_D\), and with the same quadratic expansion of \(R_{\lambda,\gamma}(w)\).

\end{proof}

\subsection{Error decomposition}\label{sec:error_decomposition}

We next establish a standard error decomposition adapted to our setting. The idea is to separate the error into a bias term and an estimation term, controlled by the deviations of the empirical covariance operators from their population counterparts.

\vspace{1em}
\begin{lemma}[Error decomposition]\label{lem:error_decomp}
Let
\[
\Sigma_\gamma := \Sigma + \gamma \Sigma_D,
\qquad
\widehat \Sigma_\gamma := \widehat \Sigma + \gamma \widehat \Sigma_D,
\qquad
V_\gamma := V + \gamma V_D,
\qquad
\widehat V_\gamma := \widehat V + \gamma \widehat V_D,
\]
and define
\[
w_{\lambda,\gamma} := (\Sigma_\gamma + \lambda I)^{-1} V_\gamma,
\qquad
\widehat w := (\widehat \Sigma_\gamma + \lambda I)^{-1} \widehat V_\gamma,
\qquad
u^*(x)=\langle w^*,\phi(x)\rangle_{\mathcal H}.
\]

Assume Assumption \ref{as:attainability} holds. Define the quantities
\begin{align*}
    &\Delta_{\Sigma,1} := \|(\Sigma_\gamma + \lambda I)^{1/2}(\widehat \Sigma_\gamma + \lambda I)^{-1/2}\big\|\\
    &\Delta_{\Sigma,2}:= \big\|(\Sigma_\gamma + \lambda I)^{-1/2}(\Sigma_\gamma - \widehat \Sigma_\gamma)
    \big\|\\
    &\Delta_V:= \big\|(\Sigma_\gamma + \lambda I)^{-1/2}(\widehat V_\gamma - V_\gamma)\big\|\\
    &\mathcal{B}(\lambda, \gamma):= \big\|(\Sigma_\gamma + \lambda I)^{-1/2} w^*\big\|.
\end{align*}
Then
\[
\|\Sigma^{1/2}(\widehat w - w^*)\|
\le
 \Delta_{\Sigma,1}^2 (\Delta_V + \big\|w^* \big\| \Delta_{\Sigma,2}) + \lambda \mathcal{B}(\lambda, \gamma).
\]
\end{lemma}

\begin{proof}
We first split the error into an estimation and a bias term:
\[
\|\Sigma^{1/2}(\widehat w - w^*)\|
\leq
\|\Sigma^{1/2}(\widehat w - w_{\lambda,\gamma})\|
+
\|\Sigma^{1/2}(w_{\lambda,\gamma} - w^*)\|.
\]
We now bound each term.

\emph{Bias term.} Under the Assumption \ref{as:attainability} we have
\[
V = \Sigma w^*,
\qquad
V_D = \Sigma_D w^*,
\]
hence
\[
V_\gamma = V + \gamma V_D = (\Sigma + \gamma \Sigma_D)w^* = \Sigma_\gamma w^*.
\]
Therefore
\[
\begin{aligned}
w_{\lambda,\gamma} - w^*
&= (\Sigma_\gamma + \lambda I)^{-1} V_\gamma - w^* \\
&= (\Sigma_\gamma + \lambda I)^{-1} \Sigma_\gamma w^* - w^* \\
&= \big[(\Sigma_\gamma + \lambda I)^{-1} \Sigma_\gamma - I\big] w^* \\
&= -\lambda (\Sigma_\gamma + \lambda I)^{-1} w^*.
\end{aligned}
\]
As a consequence,
\[
\|\Sigma^{1/2}(w_{\lambda,\gamma} - w^*)\|
= \lambda \big\|\Sigma^{1/2}(\Sigma_\gamma + \lambda I)^{-1} w^*\big\|
\leq \lambda \big\|(\Sigma_\gamma + \lambda I)^{-1/2} w^*\big\|,
\]
where we used \(\|\Sigma^{1/2}(\Sigma_\gamma + \lambda I)^{-1/2}\|\le 1\) since \(\Sigma\preccurlyeq \Sigma_\gamma + \lambda I\).

\emph{Estimation term.}
We have
\[
\widehat w - w_{\lambda,\gamma}
=
(\widehat \Sigma_\gamma + \lambda I)^{-1}(\widehat V_\gamma - V_\gamma)
+
\Big[(\widehat \Sigma_\gamma + \lambda I)^{-1} - (\Sigma_\gamma + \lambda I)^{-1}\Big]V_\gamma,
\]
and, using \(A^{-1} - B^{-1} = A^{-1}(B-A)B^{-1}\) with
\(A=\widehat \Sigma_\gamma + \lambda I\), \(B=\Sigma_\gamma + \lambda I\),
\[
(\widehat \Sigma_\gamma + \lambda I)^{-1} - (\Sigma_\gamma + \lambda I)^{-1}
=
(\widehat \Sigma_\gamma + \lambda I)^{-1}(\Sigma_\gamma - \widehat \Sigma_\gamma)(\Sigma_\gamma + \lambda I)^{-1}.
\]

Multiplying by \(\Sigma^{1/2}\) and inserting
\((\Sigma_\gamma + \lambda I)^{\pm 1/2}\), we obtain
\[
\begin{aligned}
\|\Sigma^{1/2}(\widehat w - w_{\lambda,\gamma})\|
&\le
\big\|\Sigma^{1/2}(\widehat \Sigma_\gamma + \lambda I)^{-1/2}\big\|\,
\big\|(\widehat \Sigma_\gamma + \lambda I)^{-1/2}(\Sigma_\gamma + \lambda I)^{1/2}\big\| \\
&\quad\times \Big[
\big\|(\Sigma_\gamma + \lambda I)^{-1/2}(\widehat V_\gamma - V_\gamma)\big\| \\
&\qquad\quad+
\big\|(\Sigma_\gamma + \lambda I)^{-1/2}(\Sigma_\gamma - \widehat \Sigma_\gamma)\big\|\,
\big\|(\Sigma_\gamma + \lambda I)^{-1}V_{\gamma}\big\|
\Big].
\end{aligned}
\]

Then,
\[
\begin{aligned}
\big\|\Sigma^{1/2}(\widehat \Sigma_\gamma + \lambda I)^{-1/2}\big\|
&=
\big\|\Sigma^{1/2}(\Sigma_\gamma + \lambda I)^{-1/2}
      (\Sigma_\gamma + \lambda I)^{1/2}(\widehat \Sigma_\gamma + \lambda I)^{-1/2}\big\|\\
&\le
\big\|\Sigma^{1/2}(\Sigma_\gamma + \lambda I)^{-1/2}\big\|\,
\big\|(\Sigma_\gamma + \lambda I)^{1/2}(\widehat \Sigma_\gamma + \lambda I)^{-1/2}\big\|\\
&\le
\big\|(\Sigma_\gamma + \lambda I)^{1/2}(\widehat \Sigma_\gamma + \lambda I)^{-1/2}\big\|,
\end{aligned}
\]
where we used again that \(\|\Sigma^{1/2}(\Sigma_\gamma + \lambda I)^{-1/2}\|\le 1\).

Moreover, using $V_{\gamma} = \Sigma_{\gamma}w^*$, we have
\begin{align*}
\big\|(\Sigma_\gamma + \lambda I)^{-1}V_{\gamma}\big\|
&= \big\|(\Sigma_\gamma + \lambda I)^{-1}\Sigma_{\gamma}w^*\big\|\\
&\leq \big\|w^* \big\|.
\end{align*}

\emph{Conclusion.} We define 
\begin{align*}
    &\Delta_{\Sigma,1} := \big\|(\Sigma_\gamma + \lambda I)^{1/2}(\widehat \Sigma_\gamma + \lambda I)^{-1/2}\big\|\\
    &\Delta_{\Sigma,2}:= \big\|(\Sigma_\gamma + \lambda I)^{-1/2}(\Sigma_\gamma - \widehat \Sigma_\gamma)
    \big\|\\
        &\Delta_V:= \big\|(\Sigma_\gamma + \lambda I)^{-1/2}(\widehat V_\gamma - V_\gamma)\big\|\\
    &\mathcal{B}(\lambda, \gamma):= \big\|(\Sigma_\gamma + \lambda I)^{-1/2} w^*\big\|.
\end{align*}

Collecting the terms, we obtain
\[
\|\Sigma^{1/2}(\widehat w - w_{\lambda,\gamma})\|
\;\le\; \Delta_{\Sigma,1}^2 (\Delta_V + \big\|w^* \big\| \Delta_{\Sigma,2}),
\]
and 
\[
\|\Sigma^{1/2}(w_{\lambda,\gamma} - w^*)\|
\;\le\; \lambda \mathcal{B}(\lambda, \gamma),
\]

which together yields
\[
\|\Sigma^{1/2}(\widehat w - w^*)\|
\le
 \Delta_{\Sigma,1}^2 (\Delta_V + \big\|w^* \big\| \Delta_{\Sigma,2}) + \lambda \mathcal{B}(\lambda, \gamma).
\]
\end{proof}

\subsection{High-probability bounds for the error decomposition terms
\texorpdfstring{$\Delta_{\Sigma,1}$, $\Delta_{\Sigma,2}$, and $\Delta_V$}
{Delta Sigma 1, Delta Sigma 2, and Delta V}}\label{sec:hp-bounds}

We next control the random quantities appearing in the error decomposition of Lemma~\ref{lem:error_decomp}. Compared with the standard proof for kernel ridge regression, the only additional difficulty is the presence of the derivative term and the need to keep track of its dependence on $\gamma$. This leads us to decompose both covariance and moment deviations into value and derivative contributions. The remainder of the argument then follows the standard kernel ridge regression proof strategy: we derive high-probability bounds for $\Delta_V$ and $\Delta_{\Sigma,2}$, and control the multiplicative factor $\Delta_{\Sigma,1}$ through the auxiliary normalized covariance deviation
\[
\Delta_{\Sigma,3}
:=
\big\|(\Sigma_\gamma + \lambda I)^{-1/2}(\Sigma_\gamma-\widehat\Sigma_\gamma)(\Sigma_\gamma + \lambda I)^{-1/2}\big\|.
\]
The key point is that when $\Delta_{\Sigma,3}<1$, the empirical covariance operator remains well conditioned relative to $\Sigma_\gamma+\lambda I$, which yields a bound on $\Delta_{\Sigma,1}$.

\vspace{1em}
\begin{lemma}[High-probability bound on $\Delta_{\Sigma,3}$]\label{lem:delta-sigma-3} We define
\[
\Delta_{\Sigma,3}
:=
\big\|(\Sigma_\gamma + \lambda I)^{-1/2}(\Sigma_\gamma - \widehat \Sigma_\gamma)
    (\Sigma_\gamma + \lambda I)^{-1/2}\big\|.
\]
Let $\delta\in(0,1)$ and $\lambda>0$ satisfy
\[
\frac{36}{n}\log\frac{n}{\delta} \;\le\; \lambda \;\le\; \|\Sigma\|_\infty,
\qquad
\frac{36}{m}\log\frac{m}{\delta} \;\le\; \lambda \gamma^{-1} \;\le\; \|\Sigma_D\|_\infty.
\]
Then, with probability at least $1-2\delta$,
\[
\Delta_{\Sigma,3} \;\le\; \tfrac{1}{2}.
\]
\end{lemma}

\begin{proof}
We first relate $\Delta_{\Sigma,3}$ to the value and derivative parts separately.
Using the decomposition
\[
\Sigma_\gamma - \widehat \Sigma_\gamma
=
(\Sigma - \widehat \Sigma) + \gamma (\Sigma_D - \widehat \Sigma_D),
\]
and the triangle inequality, we obtain
\begin{align*}
\Delta_{\Sigma,3}
&=
\Big\|
    (\Sigma_\gamma + \lambda I)^{-1/2}(\Sigma_\gamma - \widehat \Sigma_\gamma)
    (\Sigma_\gamma + \lambda I)^{-1/2}
\Big\|\\
&\le
\Big\|
    (\Sigma_\gamma + \lambda I)^{-1/2}(\Sigma - \widehat \Sigma)
    (\Sigma_\gamma + \lambda I)^{-1/2}
\Big\|
+
\gamma \Big\|
    (\Sigma_\gamma + \lambda I)^{-1/2}(\Sigma_D - \widehat \Sigma_D)
    (\Sigma_\gamma + \lambda I)^{-1/2}
\Big\|.
\end{align*}

For the first term, write
\[
(\Sigma_\gamma + \lambda I)^{-1/2}(\Sigma - \widehat \Sigma)
    (\Sigma_\gamma + \lambda I)^{-1/2}
=
A
\,(\Sigma + \lambda I)^{-1/2}(\Sigma - \widehat \Sigma)(\Sigma + \lambda I)^{-1/2}
\,A^*,
\]
where
\[
A:=
(\Sigma_\gamma + \lambda I)^{-1/2}(\Sigma + \lambda I)^{1/2}.
\]
From $\Sigma \preccurlyeq \Sigma_\gamma$ we have
$(\Sigma_\gamma + \lambda I)^{-1} \preccurlyeq (\Sigma + \lambda I)^{-1}$, hence
$\|A\|\le 1$. Therefore
\[
\Big\|
    (\Sigma_\gamma + \lambda I)^{-1/2}(\Sigma - \widehat \Sigma)
    (\Sigma_\gamma + \lambda I)^{-1/2}
\Big\|
\le
\big\|(\Sigma + \lambda I)^{-1/2}(\Sigma - \widehat \Sigma)(\Sigma + \lambda I)^{-1/2}\big\|.
\]

For the second term we argue analogously, but comparing with $\gamma \Sigma_D$.
Using $\gamma \Sigma_D \preccurlyeq \Sigma_\gamma$ we obtain
\[
\Big\|
    (\Sigma_\gamma + \lambda I)^{-1/2}(\Sigma_D - \widehat \Sigma_D)
    (\Sigma_\gamma + \lambda I)^{-1/2}
\Big\|
\le
\big\|(\gamma \Sigma_D + \lambda I)^{-1/2}(\Sigma_D - \widehat \Sigma_D)
    (\gamma \Sigma_D + \lambda I)^{-1/2}\big\|,
\]
and hence
\[
\Delta_{\Sigma,3}
\;\le\;
\big\|(\Sigma + \lambda I)^{-1/2}(\Sigma- \widehat \Sigma)
    (\Sigma+ \lambda I)^{-1/2}\big\|
\;+\;
\gamma \big\|(\gamma \Sigma_D + \lambda I)^{-1/2}(\Sigma_D - \widehat \Sigma_D)
    (\gamma \Sigma_D + \lambda I)^{-1/2}\big\|.
\]

We now control each term probabilistically using Lemma~3.6 in \cite{rudi2013sample}. Although the lemma yields a bound of $1/2$, this constant can be reduced by rescaling the regularization parameter and applying the same proof. In particular, applying the proof of Lemma~3.6 with target bound
\(1/4\), and using the admissibility conditions
\[
\frac{36}{n}\log\frac{n}{\delta} \;\le\; \lambda \;\le\; \|\Sigma\|_\infty,
\qquad
\gamma \frac{36}{m}\log\frac{m}{\delta} \;\le\; \lambda \;\le\; \gamma\|\Sigma_D\|_\infty,
\]
we obtain that, with probability at least $1-\delta$,
\[
\big\|(\Sigma + \lambda I)^{-1/2}(\Sigma- \widehat \Sigma)
    (\Sigma+ \lambda I)^{-1/2}\big\| \;<\; \tfrac{1}{4},
\]
and, again with probability at least $1-\delta$,
\[
\gamma \big\|(\gamma \Sigma_D + \lambda I)^{-1/2}(\Sigma_D- \widehat \Sigma_D)
    (\gamma \Sigma_D + \lambda I)^{-1/2}\big\|
\;<\; \tfrac{1}{4}.
\]
Here the improvement from \(\tfrac12\) to \(\tfrac14\) comes from running the
same concentration proof with the smaller target threshold \(\tfrac14\), at the
cost of replacing the constant \(9\) in the lower admissibility condition by
\(36\).

By a union bound, both events hold simultaneously with probability at least
$1-2\delta$, and on this event we have
\[
\Delta_{\Sigma,3}
\;\le\;
\tfrac{1}{4} + \tfrac{1}{4}
\;=\; \tfrac{1}{2}.
\]
This concludes the proof.
\end{proof}

\vspace{1em}
\begin{lemma}[Bound on $\Delta_{\Sigma,1}$]\label{lem:delta-sigma-1}
Recall
\[
\Delta_{\Sigma,1}
:= \big\|(\Sigma_\gamma + \lambda I)^{1/2}(\widehat \Sigma_\gamma + \lambda I)^{-1/2}\big\|.
\]
We have 
\[
\Delta_{\Sigma,1} \leq (1 - \Delta_{\Sigma,3})^{-1/2}.
\]
\end{lemma}
\begin{proof} We have
\[
\Delta_{\Sigma,1}^2
= \big\|(\Sigma_\gamma + \lambda I)^{1/2}(\widehat \Sigma_\gamma + \lambda I)^{-1}
(\Sigma_\gamma + \lambda I)^{1/2}\big\|.
\]

Define
\[
\widehat B
:=
(\Sigma_\gamma + \lambda I)^{-1/2}(\Sigma_\gamma - \widehat \Sigma_\gamma)
(\Sigma_\gamma + \lambda I)^{-1/2}.
\]
By definition of $\Delta_{\Sigma,3}$,
\[
\|\widehat B\| = \Delta_{\Sigma,3}.
\]

We can rewrite
\[
\widehat \Sigma_\gamma + \lambda I
= \Sigma_\gamma + \lambda I - (\Sigma_\gamma - \widehat \Sigma_\gamma)
= (\Sigma_\gamma + \lambda I)^{1/2}(I - \widehat B)(\Sigma_\gamma + \lambda I)^{1/2},
\]
so
\[
(\widehat \Sigma_\gamma + \lambda I)^{-1}
= (\Sigma_\gamma + \lambda I)^{-1/2}(I - \widehat B)^{-1}(\Sigma_\gamma + \lambda I)^{-1/2}.
\]
Plugging this into the expression for $\Delta_{\Sigma,1}^2$ gives
\[
\Delta_{\Sigma,1}^2
= \|(I - \widehat B)^{-1}\|.
\]

If $\|\widehat B\| = \Delta_{\Sigma,3} < 1$, then the spectrum of $\widehat B$ lies in
$[-\Delta_{\Sigma,3},\Delta_{\Sigma,3}]$, hence the spectrum of $(I-\widehat B)^{-1}$ is contained in $\{(1-\mu)^{-1} : |\mu|\le \Delta_{\Sigma,3}\}$, and therefore
\[
\Delta_{\Sigma,1}^2
= \|(I - \widehat B)^{-1}\|
\le (1 - \|\widehat B\|)^{-1}
= (1 - \Delta_{\Sigma,3})^{-1}.
\]
\end{proof}

\vspace{1em}
\begin{lemma}[High-probability bound on $\Delta_V$]\label{lem:delta_V}
Recall
\[
\Delta_V
:=
\big\|(\Sigma_\gamma + \lambda I)^{-1/2}(\widehat V_\gamma - V_\gamma)\big\|.
\]
Then
\[
\Delta_V
\;\le\;
\big\|(\Sigma_\gamma + \lambda I)^{-1/2}(\widehat V - V)\big\|
\;+\;
\gamma\,\big\|(\Sigma_\gamma + \lambda I)^{-1/2}(\widehat V_D - V_D)\big\|.
\]

For any $\delta\in(0,1)$, with probability at least $1-2\delta$,
\[
\Delta_V
\;\le\;
2\,\beta_V\,\frac{\log(2/\delta)}{n}
\;+\;
\sqrt{2}\,\sigma_V\,\sqrt{\frac{\log(2/\delta)}{n}}
\]
\[
\qquad
+\;
\gamma\Bigg(
2\,\beta_D\,\frac{\log(2/\delta)}{m}
+
\sqrt{2}\,\sigma_D\,\sqrt{\frac{\log(2/\delta)}{m}}
\Bigg),
\]
where the explicit constants are
\[
\beta_V := 2\,\lambda^{-1/2}L_V\kappa_V,
\qquad
\sigma_V^2 :=
L_V^2\kappa_V^2\big\|(\Sigma_\gamma + \lambda I)^{-1/2}\Sigma^{1/2}\big\|_{\mathrm{HS}}^2,
\]
\[
\beta_D := 2\,\lambda^{-1/2}L_D\kappa_D,
\qquad
\sigma_D^2 :=
L_D^2\kappa_D^2\big\|(\Sigma_\gamma + \lambda I)^{-1/2}\Sigma_D^{1/2}\big\|_{\mathrm{HS}}^2.
\]
\end{lemma}

\begin{proof}
The decomposition
\[
\widehat V_\gamma - V_\gamma
=
(\widehat V - V)
+\gamma(\widehat V_D - V_D)
\]
immediately gives
\[
\Delta_V
\le
\big\|(\Sigma_\gamma+\lambda I)^{-1/2}(\widehat V - V)\big\|
+
\gamma\big\|(\Sigma_\gamma+\lambda I)^{-1/2}(\widehat V_D - V_D)\big\|.
\]

\paragraph{Value part.}
Write
\[
(\Sigma_\gamma+\lambda I)^{-1/2}(\widehat V - V)
=
\frac{1}{n}\sum_{i=1}^n 
M(x_i,y_i),
\qquad
M(x,y):=
(\Sigma_\gamma+\lambda I)^{-1/2}(y\phi(x)-V).
\]

Since $|y|\le L_V$ and $\|\phi(x)\|\le\kappa_V$,
\[
\|y\phi(x)-V\|
\le
L_V\kappa_V + \|V\|
\le
2 L_V \kappa_V,
\]
hence
\[
\|M(x,y)\|\le \|(\Sigma_\gamma+\lambda I)^{-1/2}\|\,2L_V\kappa_V
\le
2\lambda^{-1/2}L_V\kappa_V
=: \beta_V.
\]

The second moment satisfies
\[
\mathbb{E}\|M(x,y)\|^2
\le
L_V^2\kappa_V^2
\|(\Sigma_\gamma+\lambda I)^{-1/2}\Sigma^{1/2}\|_{\mathrm{HS}}^2
=: \sigma_V^2.
\]

Applying Hilbert-space Bernstein (Prop. 7.15 of \cite{brogat2024learning}) gives,
with probability at least $1-\delta$,
\[
\big\|(\Sigma_\gamma+\lambda I)^{-1/2}(\widehat V-V)\big\|
\le
2\beta_V\,\frac{\log(2/\delta)}{n}
+
\sqrt{2}\,\sigma_V\sqrt{\frac{\log(2/\delta)}{n}}.
\]

\paragraph{Derivative part.}
Exactly the same argument with 
\[
M_D(z,d)
:=
(\Sigma_\gamma+\lambda I)^{-1/2}(d\,D\phi(z)-V_D)
\]
yields
\[
\|M_D(z,d)\|\le 
2\lambda^{-1/2}L_D\kappa_D
=:\beta_D,
\]
and
\[
\sigma_D^2=
L_D^2\kappa_D^2
\|(\Sigma_\gamma+\lambda I)^{-1/2}\Sigma_D^{1/2}\|_{\mathrm{HS}}^2.
\]

Bernstein again gives, with probability at least $1-\delta$,
\[
\big\|(\Sigma_\gamma+\lambda I)^{-1/2}(\widehat V_D-V_D)\big\|
\le
2\,\beta_D\,\frac{\log(2/\delta)}{m}
+
\sqrt{2}\,\sigma_D\sqrt{\frac{\log(2/\delta)}{m}}.
\]

\paragraph{Conclusion.}
A union bound yields the stated bound with probability at least $1-2\delta$.
\end{proof}

\vspace{1em}
\begin{lemma}[High-probability bound on $\Delta_{\Sigma,2}$]\label{lem:delta-sigma-2}
Recall
\[
\Delta_{\Sigma,2}
:=
\big\|(\Sigma_\gamma + \lambda I)^{-1/2}(\Sigma_\gamma - \widehat \Sigma_\gamma)\big\|,
\qquad
\Sigma_\gamma = \Sigma + \gamma \Sigma_D,\quad
\widehat \Sigma_\gamma = \widehat \Sigma + \gamma \widehat \Sigma_D.
\]
Then
\[
\Delta_{\Sigma,2}
\;\le\;
\big\|(\Sigma_\gamma + \lambda I)^{-1/2}(\Sigma - \widehat \Sigma)\big\|
\;+\;
\gamma\,\big\|(\Sigma_\gamma + \lambda I)^{-1/2}(\Sigma_D - \widehat \Sigma_D)\big\|.
\]

Moreover, for any $\delta\in(0,1)$, with probability at least $1-2\delta$,
\[
\Delta_{\Sigma,2}
\;\le\;
2\,\beta_\Sigma\,\frac{\log(2/\delta)}{n}
\;+\;
\sqrt{2}\,\sigma_\Sigma\,\sqrt{\frac{\log(2/\delta)}{n}}
\;+\;
\gamma\Bigg(
2\,\beta_{\Sigma_D}\,\frac{\log(2/\delta)}{m}
\;+\;
\sqrt{2}\,\sigma_{\Sigma_D}\,\sqrt{\frac{\log(2/\delta)}{m}}
\Bigg),
\]
where
\[
\beta_\Sigma := \lambda^{-1/2}\,2\kappa_V^2,
\qquad
\sigma_\Sigma^2 := 4\kappa_V^2\,\big\|(\Sigma_\gamma + \lambda I)^{-1/2}\Sigma^{1/2}\big\|_{\mathrm{HS}}^2,
\]
\[
\beta_{\Sigma_D} := \lambda^{-1/2}\,2\kappa_D^2,
\qquad
\sigma_{\Sigma_D}^2 := 4\kappa_D^2\,\big\|(\Sigma_\gamma + \lambda I)^{-1/2}\Sigma_D^{1/2}\big\|_{\mathrm{HS}}^2.
\]
\end{lemma}

\begin{proof}
The decomposition
\[
\Sigma_\gamma - \widehat \Sigma_\gamma
=
(\Sigma - \widehat \Sigma) + \gamma(\Sigma_D - \widehat \Sigma_D)
\]
and the triangle inequality give
\[
\begin{aligned}
\Delta_{\Sigma,2}
&=
\big\|(\Sigma_\gamma + \lambda I)^{-1/2}(\Sigma_\gamma - \widehat \Sigma_\gamma)\big\|\\
&\le
\big\|(\Sigma_\gamma + \lambda I)^{-1/2}(\Sigma - \widehat \Sigma)\big\|
+
\gamma\,\big\|(\Sigma_\gamma + \lambda I)^{-1/2}(\Sigma_D - \widehat \Sigma_D)\big\|.
\end{aligned}
\]
We now bound these two terms separately, with the same Hilbert-space Bernstein
argument as in Lemma~\ref{lem:delta_V}, but now in the Hilbert space of Hilbert-Schmidt operators and with $y\phi(x)$ replaced by $\phi(x)\otimes\phi(x)$.

\paragraph{Value part.}
Recall
\[
\Sigma = \mathbb E[\phi(x)\otimes\phi(x)],
\qquad
\widehat \Sigma = \frac1n\sum_{i=1}^n \phi(x_i)\otimes\phi(x_i).
\]
Define
\[
U(x)
:=
\Sigma - \phi(x)\otimes\phi(x),
\qquad
M_\Sigma(x)
:=
(\Sigma_\gamma + \lambda I)^{-1/2}U(x).
\]
Then $\mathbb E[U(x)]=0$, and
\[
(\Sigma_\gamma + \lambda I)^{-1/2}(\Sigma - \widehat \Sigma)
=
\frac1n\sum_{i=1}^n M_\Sigma(x_i).
\]

We view $M_\Sigma(x)$ as a random element of the Hilbert space of
Hilbert-Schmidt operators. We first bound its norm. Using the boundedness
$\|\phi(x)\|\le \kappa_V$ we have
\[
\|\phi(x)\otimes\phi(x)\|_{\mathrm{HS}} = \|\phi(x)\|^2 \le \kappa_V^2,
\qquad
\|\Sigma\|_{\mathrm{HS}} \le \mathbb E\|\phi(x)\otimes\phi(x)\|_{\mathrm{HS}}
\le \kappa_V^2.
\]
Thus
\[
\|U(x)\|_{\mathrm{HS}}
\le \|\phi(x)\otimes\phi(x)\|_{\mathrm{HS}} + \|\Sigma\|_{\mathrm{HS}}
\le 2\kappa_V^2,
\]
and since $\|(\Sigma_\gamma + \lambda I)^{-1/2}\|\le\lambda^{-1/2}$,
\[
\|M_\Sigma(x)\|_{\mathrm{HS}}
\le \lambda^{-1/2}\,2\kappa_V^2
= \beta_\Sigma.
\]

Next, we bound the second moment. Using that $\Sigma_\gamma+\lambda I$ is
self-adjoint and positive, we have
\[
\|M_\Sigma(x)\|_{\mathrm{HS}}^2
=
\big\|(\Sigma_\gamma + \lambda I)^{-1/2}U(x)\big\|_{\mathrm{HS}}^2
=
\mathrm{Tr}\Big(
U(x)(\Sigma_\gamma + \lambda I)^{-1}U(x)
\Big).
\]
Take expectation and use linearity of the trace:
\[
\mathbb E\|M_\Sigma(x)\|_{\mathrm{HS}}^2
=
\mathrm{Tr}\Big(
(\Sigma_\gamma + \lambda I)^{-1}\,\mathbb E\big[U(x)^2\big]
\Big).
\]

We now bound $\mathbb E[U(x)^2]$ in terms of $\Sigma$. First note that
\[
U(x)^2
=
\big(\Sigma - \phi(x)\otimes\phi(x)\big)^2
\preccurlyeq
2\Sigma^2 + 2\big(\phi(x)\otimes\phi(x)\big)^2,
\]
by the elementary operator inequality
$(A-B)^2 = A^2 + B^2 - AB - BA \preccurlyeq 2A^2 + 2B^2$ for self-adjoint $A,B$.

Moreover,
\[
\big(\phi(x)\otimes\phi(x)\big)^2
=
\|\phi(x)\|^2\,\phi(x)\otimes\phi(x)
\preccurlyeq
\kappa_V^2\,\phi(x)\otimes\phi(x),
\]
and $\Sigma^2 \preccurlyeq \|\Sigma\|\,\Sigma \preccurlyeq \kappa_V^2\,\Sigma$, since
$\|\Sigma\|\le\kappa_V^2$ for a bounded feature map. Combining these,
\[
U(x)^2
\preccurlyeq
2\kappa_V^2\,\phi(x)\otimes\phi(x)
+
2\kappa_V^2\,\Sigma.
\]
Taking expectations yields
\[
\mathbb E[U(x)^2]
\preccurlyeq
2\kappa_V^2\,\mathbb E[\phi(x)\otimes\phi(x)]
+
2\kappa_V^2\,\Sigma
=
4\kappa_V^2\,\Sigma.
\]

Plugging this into the expression for the second moment, we get
\[
\begin{aligned}
\mathbb E\|M_\Sigma(x)\|_{\mathrm{HS}}^2
&\le
\mathrm{Tr}\Big(
(\Sigma_\gamma + \lambda I)^{-1}\,4\kappa_V^2 \Sigma
\Big)\\
&=
4\kappa_V^2\,
\mathrm{Tr}\Big(
(\Sigma_\gamma + \lambda I)^{-1/2}\Sigma(\Sigma_\gamma + \lambda I)^{-1/2}
\Big)\\
&=
4\kappa_V^2\,
\big\|(\Sigma_\gamma + \lambda I)^{-1/2}\Sigma^{1/2}\big\|_{\mathrm{HS}}^2.
\end{aligned}
\]
Thus we set
\[
\sigma_\Sigma^2
:=
4\kappa_V^2\,\big\|(\Sigma_\gamma + \lambda I)^{-1/2}\Sigma^{1/2}\big\|_{\mathrm{HS}}^2
\]
and obtain
\[
\mathbb E\|M_\Sigma(x)\|_{\mathrm{HS}}^2 \le \sigma_\Sigma^2.
\]

We can now apply the Hilbert-space Bernstein inequality (Proposition~7.15
in~\cite{brogat2024learning}) in the Hilbert space of Hilbert-Schmidt operators to the empirical average $\frac1n\sum_{i=1}^n M_\Sigma(x_i)$. This gives: for any $\delta\in(0,1)$, with
probability at least $1-\delta$,
\[
\big\|(\Sigma_\gamma + \lambda I)^{-1/2}(\Sigma - \widehat \Sigma)\big\|_{\mathrm{HS}}
\;\le\;
2\,\beta_\Sigma\,\frac{\log(2/\delta)}{n}
+
\sqrt{2}\,\sigma_\Sigma\,\sqrt{\frac{\log(2/\delta)}{n}}.
\]
Since \(\|T\|\le \|T\|_{\mathrm{HS}}\) for every Hilbert-Schmidt operator \(T\), the same bound holds for
\[
\big\|(\Sigma_\gamma + \lambda I)^{-1/2}(\Sigma - \widehat \Sigma)\big\|.
\]

\paragraph{Derivative part.}
The argument is identical with
\[
\Sigma_D
=
\mathbb E[D\phi(z)\otimes D\phi(z)],
\qquad
\widehat \Sigma_D
=
\frac1m\sum_{j=1}^m D\phi(z_j)\otimes D\phi(z_j).
\]
Define
\[
U_D(z)
:=
\Sigma_D - D\phi(z)\otimes D\phi(z),
\qquad
M_{\Sigma_D}(z)
:=
(\Sigma_\gamma + \lambda I)^{-1/2}U_D(z).
\]
Then
\[
(\Sigma_\gamma + \lambda I)^{-1/2}(\Sigma_D - \widehat \Sigma_D)
=
\frac1m\sum_{j=1}^m M_{\Sigma_D}(z_j),
\]
with $\mathbb E[M_{\Sigma_D}(z)]=0$.

Using $\|D\phi(z)\|\le\kappa_D$ we obtain
\[
\|U_D(z)\|_{\mathrm{HS}}\le 2\kappa_D^2,
\qquad
\|M_{\Sigma_D}(z)\|_{\mathrm{HS}}\le \lambda^{-1/2}2\kappa_D^2 = \beta_{\Sigma_D},
\]
and, by the same operator inequalities as above,
\[
\mathbb E\|M_{\Sigma_D}(z)\|_{\mathrm{HS}}^2
\le
4\kappa_D^2\,
\big\|(\Sigma_\gamma + \lambda I)^{-1/2}\Sigma_D^{1/2}\big\|_{\mathrm{HS}}^2
=: \sigma_{\Sigma_D}^2.
\]
Another application of the Hilbert-space Bernstein inequality in the Hilbert space of Hilbert-Schmidt operators yields, with
probability at least $1-\delta$,
\[
\big\|(\Sigma_\gamma + \lambda I)^{-1/2}(\Sigma_D - \widehat \Sigma_D)\big\|_{\mathrm{HS}}
\;\le\;
2\,\beta_{\Sigma_D}\,\frac{\log(2/\delta)}{m}
+
\sqrt{2}\,\sigma_{\Sigma_D}\,\sqrt{\frac{\log(2/\delta)}{m}}.
\]
Again the same bound holds for
\[
\big\|(\Sigma_\gamma + \lambda I)^{-1/2}(\Sigma_D - \widehat \Sigma_D)\big\|.
\]

\paragraph{Conclusion.}
We have shown that, with probability at least $1-\delta$,
\[
\big\|(\Sigma_\gamma + \lambda I)^{-1/2}(\Sigma - \widehat \Sigma)\big\|
\le
2\,\beta_\Sigma\,\frac{\log(2/\delta)}{n}
+
\sqrt{2}\,\sigma_\Sigma\,\sqrt{\frac{\log(2/\delta)}{n}},
\]
and, with probability at least $1-\delta$,
\[
\big\|(\Sigma_\gamma + \lambda I)^{-1/2}(\Sigma_D - \widehat \Sigma_D)\big\|
\le
2\,\beta_{\Sigma_D}\,\frac{\log(2/\delta)}{m}
+
\sqrt{2}\,\sigma_{\Sigma_D}\,\sqrt{\frac{\log(2/\delta)}{m}}.
\]
By a union bound, both events hold simultaneously with probability at least
$1-2\delta$, and combining these with the initial decomposition of $\Delta_{\Sigma,2}$ yields the claimed high-probability bound.
\end{proof}

\subsection{Finite-sample bounds}\label{sec:finite_sample}

Combining Lemma~\ref{lem:error_decomp} with the concentration results of the previous subsection yields the following finite-sample bound.

\begin{theorem}[High-probability finite-sample bounds for PIKS] Let $\lambda>0$, $\gamma>0$ and let $\delta\in(0,1/2)$. Assume that
\begin{equation}
\frac{36}{n}\log\frac{2n}{\delta} \;\le\; \lambda \;\le\; \|\Sigma\|_\infty,
\qquad
\frac{36}{m}\log\frac{2m}{\delta} \;\le\; \lambda \gamma^{-1} \;\le\; \|\Sigma_D\|_\infty.
\end{equation}
Define
\[
\Sigma_\gamma := \Sigma + \gamma \Sigma_D,
\qquad
\mathcal{B}(\lambda,\gamma):=\big\|(\Sigma_\gamma + \lambda I)^{-1/2}w^*\big\|,
\]
and the (value and derivative) capacity terms
\[
\mathcal{N}_V(\lambda,\gamma)
:=
\big\|(\Sigma_\gamma + \lambda I)^{-1/2}\Sigma^{1/2}\big\|_{\mathrm{HS}},
\qquad
\mathcal{N}_D(\lambda,\gamma)
:=
\big\|(\Sigma_\gamma + \lambda I)^{-1/2}\Sigma_D^{1/2}\big\|_{\mathrm{HS}}.
\]

Then there exists a constant $c>0$, depending only on $\kappa_V,\kappa_D,L_V,L_D$ and $\|w^*\|$, but not on $\lambda,\gamma,n,m,\delta$, such that, with probability at least $1-2\delta$,
\begin{equation}
\big\|\Sigma^{1/2}(\widehat w - w^*)\big\|
\;\le\;
c\Bigg[
\lambda^{-1/2}\Big(
\frac{\log(4/\delta)}{n}
+
\gamma\,\frac{\log(4/\delta)}{m}
\Big)
+
\Big(
\frac{\mathcal{N}_V(\lambda,\gamma)}{\sqrt{n}}
+
\gamma\,\frac{\mathcal{N}_D(\lambda,\gamma)}{\sqrt{m}}
\Big)\sqrt{\log\frac{4}{\delta}}
+
\lambda\,\mathcal{B}(\lambda,\gamma)
\Bigg].
\end{equation}
\end{theorem}

\begin{proof}
From the error decomposition Lemma~\ref{lem:error_decomp}, we have
\[
\|\Sigma^{1/2}(\widehat w - w^*)\|
\;\le\;
\Delta_{\Sigma,1}^2\big(\Delta_V + \|w^*\|\Delta_{\Sigma,2}\big)
\;+\;
\lambda\,\mathcal{B}(\lambda,\gamma),
\]
where
\[
\Delta_{\Sigma,1}
:= \big\|(\Sigma_\gamma + \lambda I)^{1/2}(\widehat \Sigma_\gamma + \lambda I)^{-1/2}\big\|,
\quad
\Delta_{\Sigma,2}
:= \big\|(\Sigma_\gamma + \lambda I)^{-1/2}(\Sigma_\gamma - \widehat \Sigma_\gamma)\big\|,
\]
\[
\Delta_V
:= \big\|(\Sigma_\gamma + \lambda I)^{-1/2}(\widehat V_\gamma - V_\gamma)\big\|.
\]

\emph{Step 1: bound on $\Delta_{\Sigma,1}$.}
Recall
\[
\Delta_{\Sigma,3}
:=
\big\|(\Sigma_\gamma + \lambda I)^{-1/2}(\Sigma_\gamma - \widehat \Sigma_\gamma)
    (\Sigma_\gamma + \lambda I)^{-1/2}\big\|.
\]
Applying Lemma~\ref{lem:delta-sigma-3} with confidence parameter \(\eta\), and the condition \eqref{eq:lambda-cond}, we obtain an event of probability at least \(1-2\eta\) on which \(\Delta_{\Sigma,3}\le 1/2\). On this event, Lemma~\ref{lem:delta-sigma-1} yields
\[
\Delta_{\Sigma,1}^2
\le
(1-\Delta_{\Sigma,3})^{-1}
\le 2.
\]

\emph{Step 2: bounds on $\Delta_V$ and $\Delta_{\Sigma,2}$.}
We use Lemmas~\ref{lem:delta_V} and~\ref{lem:delta-sigma-2}.

From Lemma~\ref{lem:delta_V}, for any $\eta\in(0,1)$, with probability at
least $1-2\eta$,
\[
\Delta_V
\;\le\;
\underbrace{
\Big(
2\beta_V\,\frac{\log(2/\eta)}{n}
+\sqrt{2}\sigma_V\,\sqrt{\frac{\log(2/\eta)}{n}}
\Big)
}_{\text{value part}}
+
\gamma\underbrace{
\Big(
2\beta_{D}\,\frac{\log(2/\eta)}{m}
+\sqrt{2}\sigma_D\,\sqrt{\frac{\log(2/\eta)}{m}}
\Big)
}_{\text{derivative part}},
\]
with
\[
\beta_{V} = \lambda^{-1/2} 2 L_V \kappa_V,
\qquad
\sigma_V^2 =
L_V^2\,\kappa_V^2 \big\|(\Sigma_\gamma + \lambda I)^{-1/2}\Sigma^{1/2}\big\|_{\text{HS}}^2,
\]
\[
\beta_{D} = \lambda^{-1/2} 2 L_D \kappa_D,
\qquad
\sigma_D^2 =
L_D^2\,\kappa_D^2 \big\|(\Sigma_\gamma + \lambda I)^{-1/2}\Sigma_D^{1/2}\big\|_{\text{HS}}^2.
\]

Hence there exist constants $a_{V},b_{V},a_{D},b_{D}>0$ depending only on $(\kappa_V,L_V)$ and $(\kappa_D,L_D)$ such that
\[
\Delta_V
\;\le\;
\lambda^{-1/2}\Big(
a_V\,\frac{\log(2/\eta)}{n}
+
\gamma\,a_D\,\frac{\log(2/\eta)}{m}
\Big)
+
\Big(
b_V\,\mathcal{N}_V(\lambda,\gamma)\,\sqrt{\frac{\log(2/\eta)}{n}}
+
\gamma\,b_D\,\mathcal{N}_D(\lambda,\gamma)\,\sqrt{\frac{\log(2/\eta)}{m}}
\Big).
\]

Similarly, from Lemma~\ref{lem:delta-sigma-2}, for the same confidence parameter,
with probability at least $1-2\eta$,
\[
\Delta_{\Sigma,2}
\;\le\;
\lambda^{-1/2}\Big(
\tilde a_V\,\frac{\log(2/\eta)}{n}
+
\gamma\,\tilde a_D\,\frac{\log(2/\eta)}{m}
\Big)
+
\Big(
\tilde b_V\,\mathcal{N}_V(\lambda,\gamma)\,\sqrt{\frac{\log(2/\eta)}{n}}
+
\gamma\,\tilde b_D\,\mathcal{N}_D(\lambda,\gamma)\,\sqrt{\frac{\log(2/\eta)}{m}}
\Big),
\]
for some constants $\tilde a_V,\tilde a_D,\tilde b_V,\tilde b_D>0$ depending only on $\kappa_V$ and $\kappa_D$.

Combining these two bounds, we obtain constants $A_1,A_2,B_1,B_2>0$, depending only on $(\kappa_V,\kappa_D,L_V,L_D)$, such that
\begin{equation}\label{eq:DeltaV-bound}
\Delta_V
\;\le\;
\lambda^{-1/2}\Big(
A_1\,\frac{\log(2/\eta)}{n}
+
A_2\,\gamma\,\frac{\log(2/\eta)}{m}
\Big)
+
\Big(
B_1\,\mathcal{N}_V(\lambda,\gamma)\,\sqrt{\frac{\log(2/\eta)}{n}}
+
B_2\,\gamma\,\mathcal{N}_D(\lambda,\gamma)\,\sqrt{\frac{\log(2/\eta)}{m}}
\Big),
\end{equation}
and
\begin{equation}\label{eq:DeltaC2-bound}
\Delta_{\Sigma,2}
\;\le\;
\lambda^{-1/2}\Big(
A_1\,\frac{\log(2/\eta)}{n}
+
A_2\,\gamma\,\frac{\log(2/\eta)}{m}
\Big)
+
\Big(
B_1\,\mathcal{N}_V(\lambda,\gamma)\,\sqrt{\frac{\log(2/\eta)}{n}}
+
B_2\,\gamma\,\mathcal{N}_D(\lambda,\gamma)\,\sqrt{\frac{\log(2/\eta)}{m}}
\Big),
\end{equation}
where we have simply taken maxima of the corresponding constants from $\Delta_V$ and $\Delta_{\Sigma,2}$.

\emph{Step 3: combine all bounds.}
Set $\eta := \delta/3$ and apply Lemmas~\ref{lem:delta-sigma-3},
\ref{lem:delta_V}, and \ref{lem:delta-sigma-2} with confidence parameter $\eta$.
Each holds with probability at least $1-2\eta$, so by a union bound, with
probability at least
\[
1 - 3\cdot 2\eta \;=\; 1-2\delta,
\]
we simultaneously have
\[
\Delta_{\Sigma,1}^2\le 2,
\quad
\text{and the bounds \eqref{eq:DeltaV-bound}-\eqref{eq:DeltaC2-bound} for }
\Delta_V,\Delta_{\Sigma,2}.
\]

On this event,
\[
\begin{aligned}
\|\Sigma^{1/2}(\widehat w - w^*)\|
&\le
\Delta_{\Sigma,1}^2\big(\Delta_V + \|w^*\|\Delta_{\Sigma,2}\big)
+ \lambda\,\mathcal{B}(\lambda,\gamma)\\
&\le
2\big(\Delta_V + \|w^*\|\Delta_{\Sigma,2}\big)
+ \lambda\,\mathcal{B}(\lambda,\gamma).
\end{aligned}
\]
Substituting the bounds \eqref{eq:DeltaV-bound}-\eqref{eq:DeltaC2-bound}, we obtain
\[
\begin{aligned}
\|\Sigma^{1/2}(\widehat w - w^*)\|
&\le
2(1+\|w^*\|)\Bigg[
\lambda^{-1/2}\Big(
A_1\,\frac{\log(2/\eta)}{n}
+
A_2\,\gamma\,\frac{\log(2/\eta)}{m}
\Big)\\
&\hspace{5em}+
\Big(
B_1\,\mathcal{N}_V(\lambda,\gamma)\,\sqrt{\frac{\log(2/\eta)}{n}}
+
B_2\,\gamma\,\mathcal{N}_D(\lambda,\gamma)\,\sqrt{\frac{\log(2/\eta)}{m}}
\Big)
\Bigg]
+ \lambda\,\mathcal{B}(\lambda,\gamma).
\end{aligned}
\]

Since $\eta=\delta/3$, we have $\log(2/\eta)=\log(6/\delta)\le c_1\log(4/\delta)$
for some constant $c_1>0$. Absorbing the constants $2(1+\|w^*\|)A_1, 2(1+\|w^*\|)A_2, 2(1+\|w^*\|)B_1, 2(1+\|w^*\|)B_2$ and $\sqrt{c_1}$ into a single constant $c>0$ depending only on $\kappa_V,\kappa_D,L_V,L_D$ and $\|w^*\|$, we rewrite the bound as
\[
\|\Sigma^{1/2}(\widehat w - w^*)\|
\;\le\;
c\Bigg[
\lambda^{-1/2}\Big(
\frac{\log(4/\delta)}{n}
+
\gamma\,\frac{\log(4/\delta)}{m}
\Big)
+
\Big(
\frac{\mathcal{N}_V(\lambda,\gamma)}{\sqrt{n}}
+
\gamma\,\frac{\mathcal{N}_D(\lambda,\gamma)}{\sqrt{m}}
\Big)\sqrt{\log\frac{4}{\delta}}
+
\lambda\,\mathcal{B}(\lambda,\gamma)
\Bigg],
\]
which is exactly \eqref{eq:main-bound}.
\end{proof}

\subsection{Learning rate acceleration from differential information}
\label{sec:learning_rates}

We now instantiate the finite-sample bound under the Assumption \ref{ass:abstract-alignment} and optimize over $\lambda, \gamma$.
\vspace{1em}
\begin{corollary}[Learning rates for PIKS]
Assume Assumptions~\ref{as:D_order}-\ref{ass:abstract-alignment} hold. Then there exist choices of regularization parameters $\lambda=\lambda(n,m)$ and $\gamma=\gamma(n,m)$, given explicitly in the proof and depending polynomially on $n$, $m$, $\log n$, and $\log m$, such that, with high probability:
\[
\|\Sigma^{1/2}(\widehat w-w^*)\|
\;\lesssim\;
\begin{cases}
n^{-\frac14}m^{-\frac{1-r}{8}}, & m \lesssim n^{\frac{2(1-\alpha)}{(1+\alpha)(1-r)}},\\[6pt]
n^{-\frac{1}{2(1+\alpha)}}, & m \gtrsim n^{\frac{2(1-\alpha)}{(1+\alpha)(1-r)}}.
\end{cases}
\]
\end{corollary}
\begin{proof}
The proof proceeds in four steps.

\emph{Step 1: Capacity bounds.} As a consequence of Assumption~\ref{ass:abstract-alignment}, using $\Sigma_\gamma+\lambda I \succeq \Sigma_1+\lambda I$ and $\Sigma_\gamma+\lambda I \succeq \gamma \Sigma_D+\lambda I$, we obtain
\begin{align}\label{eq:abstract-CV}
\mathcal N_V(\lambda,\gamma)^2
&=
\Tr\!\big(\Sigma_1(\Sigma_\gamma+\lambda I)^{-1}\big)
+
\Tr\!\big(\Sigma_2(\Sigma_\gamma+\lambda I)^{-1}\big)
\;\lesssim\;
\lambda^{-\alpha}
+
\gamma^{\,r-1}\lambda^{-r}.
\end{align}
Moreover, since $\Sigma_\gamma+\lambda I \succeq \lambda I$,
\[
\mathcal N_D(\lambda,\gamma)^2
\;\le\;
\lambda^{-1}\Tr(\Sigma_D)
\;\lesssim\;
\lambda^{-1}.
\]

\emph{Step 2: Error bound.}
From Assumption~\ref{as:attainability}, the model is well specified, i.e.
\(
\|w^*\|_{\mathcal H}<+\infty
\). Then
\[
\mathcal B(\lambda,\gamma)^2
=
\langle w^*,(\Sigma_\gamma+\lambda I)^{-1}w^*\rangle
\;\le\;
\lambda^{-1}\|w^*\|_{\mathcal H}^2,
\qquad
\lambda\,\mathcal B(\lambda,\gamma)
\;\lesssim\;
\sqrt{\lambda}.
\]

Throughout the remainder of the proof, we restrict to regularization parameters
satisfying
\[
\max\!\left\{\frac{\log n}{n},\ \gamma\,\frac{\log m}{m}\right\}
\;\lesssim\;
\lambda
\;\lesssim\;
\gamma .
\]
This ensures the admissibility condition of Theorem~\ref{thm:finite-sample}. Under this regime, using $\sqrt{a+b}\le \sqrt a+\sqrt b$, the finite-sample bound simplifies to
\begin{equation}\label{eq:start-opt}
\|\Sigma^{1/2}(\widehat w-w^*)\|
\;\lesssim\;
\frac{\lambda^{-\alpha/2}}{\sqrt n}
\;+\;
\frac{\gamma^{\frac{r-1}{2}}\lambda^{-r/2}}{\sqrt n}
\;+\;
\frac{\gamma\,\lambda^{-1/2}}{\sqrt m}
\;+\;
\sqrt{\lambda}.
\end{equation}

This bound makes explicit how derivative information can reduce the effective
value complexity through the aligned component $\Sigma_2$, while simultaneously
introducing a variance cost that grows with $\gamma$ through the derivative
observations.

\emph{Step 3: Optimization over $\gamma$.}
For fixed $\lambda$, only the middle two terms in \eqref{eq:start-opt} depend on $\gamma$. Increasing $\gamma$ reduces the variance associated with estimating the function values in directions constrained by the derivative operator (by shrinking the contribution of the corresponding eigenspaces), but simultaneously amplifies the variance coming from the noisy estimation of derivative information. The optimal choice of $\gamma$ is obtained by balancing these two effects by solving
\[
\frac{\gamma^{\frac{r-1}{2}}\lambda^{-r/2}}{\sqrt n}
\;=\;
\frac{\gamma\,\lambda^{-1/2}}{\sqrt m}.
\]
This is equivalent to
\[
\gamma^{\frac{r-3}{2}}
\;=\;
\left(\frac{n}{m}\right)^{1/2}\lambda^{\frac{r-1}{2}},
\]
which yields
\begin{equation}\label{eq:gamma-star}
\gamma^\star(\lambda)
\;\asymp\;
\left(\frac{m}{n}\right)^{\!\frac{1}{3-r}}
\lambda^{\frac{1-r}{3-r}}.
\end{equation}
Substituting \eqref{eq:gamma-star} into \eqref{eq:start-opt}, the two
$\gamma$–dependent terms become equal and we obtain
\begin{equation}\label{eq:lambda-only-bound}
\|\Sigma^{1/2}(\widehat w-w^*)\|
\;\lesssim\;
A(\lambda) \;+\; B(\lambda) \;+\; \sqrt{\lambda},
\end{equation}
where
\[
A(\lambda)
:=
\frac{\lambda^{-\alpha/2}}{\sqrt n},
\qquad
B(\lambda)
:=
m^{-\frac{1-r}{2(3-r)}}\,n^{-\frac{1}{3-r}}\,
\lambda^{-\frac{1+r}{2(3-r)}}.
\]

Substituting $\gamma^\star(\lambda)\asymp (m/n)^{\frac1{3-r}}\lambda^{\frac{1-r}{3-r}}$ into the working regime $\max\{\log n/n,\gamma\log m/m\}\lesssim \lambda \lesssim \gamma$ yields the equivalent $\lambda$-only conditions
\[
\max\!\left\{\frac{\log n}{n},\ \Big(\frac{m}{n}\Big)^{\!\frac12}
\Big(\frac{\log m}{m}\Big)^{\!\frac{3-r}{2}}\right\}
\;\lesssim\;
\lambda
\;\lesssim\;
\sqrt{\frac{m}{n}}.
\]

\emph{Step 4: Optimization over $\lambda$ and regime analysis.} We choose $\lambda$ by balancing the bias $\sqrt{\lambda}$ with one of the variance terms in \eqref{eq:lambda-only-bound}.

\medskip
\noindent\textbf{(i) Balance with $A(\lambda)$.}
We solve
\[
\sqrt{\lambda}
\;=\;
\frac{\lambda^{-\alpha/2}}{\sqrt n},
\]
which gives
\begin{equation}\label{eq:lambda-A}
\lambda_A^\star \;\asymp\; n^{-\frac{1}{1+\alpha}},
\qquad
\sqrt{\lambda_A^\star}
\;\asymp\;
A(\lambda_A^\star)
\;\asymp\;
n^{-\frac{1}{2(1+\alpha)}}.
\end{equation}

\medskip
\noindent\textbf{(ii) Balance with $B(\lambda)$.}
We solve
\[
\sqrt{\lambda}
\;=\;
m^{-\frac{1-r}{2(3-r)}}\,n^{-\frac{1}{3-r}}\,
\lambda^{-\frac{1+r}{2(3-r)}},
\]
which gives
\begin{equation}\label{eq:lambda-B}
\lambda_B^\star
\;\asymp\;
m^{-\frac{1-r}{4}}\,n^{-\frac12},
\qquad
\sqrt{\lambda_B^\star}
\;\asymp\;
B(\lambda_B^\star)
\;\asymp\;
n^{-\frac14}m^{-\frac{1-r}{8}}.
\end{equation}

\paragraph{Validity of the regimes.}
The bound \eqref{eq:lambda-only-bound} is controlled by the dominant variance term
at the chosen $\lambda$.

\medskip
\noindent
\emph{Saturation regime.}
We choose $\lambda=\lambda_A^\star$ when
\[
B(\lambda_A^\star)\;\lesssim\;A(\lambda_A^\star).
\]
Substituting \(\lambda_A^\star\asymp n^{-1/(1+\alpha)}\) into \(B(\lambda)\) gives
\[
B(\lambda_A^\star)
\asymp
m^{-\frac{1-r}{2(3-r)}}
n^{-\frac{1+2\alpha-r}{2(1+\alpha)(3-r)}}.
\]
Hence
\[
B(\lambda_A^\star)\lesssim A(\lambda_A^\star)
\]
is equivalent to
\[
m
\gtrsim
n^{\frac{2(1-\alpha)}{(1+\alpha)(1-r)}}.
\]

\medskip
\noindent
\emph{Derivative-limited regime.}
We choose $\lambda=\lambda_B^\star$ when
\[
A(\lambda_B^\star)\;\lesssim\;B(\lambda_B^\star).
\]
which yields the same transition value for $m$.

Moreover, to ensure our working regime condition, we take the regularization parameter to be clipped as
\[
\lambda := \max\!\left\{\lambda^\star,\; \frac{\log n}{n},\; \gamma^\star(\lambda^\star) \frac{\log m}{m}\right\},
\]
where $\lambda^\star\in\{\lambda_A^\star,\lambda_B^\star\}$ is the optimizer in the
corresponding regime. One can check that this choice of \(\lambda\) satisfies the admissibility conditions, including the upper condition \(\lambda\lesssim\gamma\), in both regimes. Since the variance terms are nonincreasing in \(\lambda\), clipping only increases the bias term, yielding an additional lower-order contribution of order \(\sqrt{\log n/n}+\sqrt{\gamma \log m/m}\).

\paragraph{Resulting rate.}
Let
\[
m_{\mathrm{crit}}
\;\asymp\;
n^{\frac{2(1-\alpha)}{(1+\alpha)(1-r)}}.
\]
The optimized learning rate is
\[
\|\Sigma^{1/2}(\widehat w-w^*)\|
\;\lesssim\;
\begin{cases}
n^{-\frac14}m^{-\frac{1-r}{8}}, & m \lesssim m_{\mathrm{crit}} ,\\[6pt]
n^{-\frac{1}{2(1+\alpha)}}, & m \gtrsim m_{\mathrm{crit}}.
\end{cases}
\]
\end{proof}

\subsection{Learning rates in physically consistent norms}

We finally show that the estimator also converges in the stronger norm induced by \(\Sigma+\Sigma_D\), which jointly controls the prediction error and the error on the differential quantities.

\vspace{1em}
\begin{corollary}[Learning rates in the physically consistent norm]
Assume Assumptions~\ref{as:D_order}-\ref{ass:abstract-alignment} hold. Then there exist choices of regularization parameters $\lambda=\lambda(n,m)$ and $\gamma=\gamma(n,m)$ such that, with high probability,
\[
\|(\Sigma+\Sigma_D)^{1/2}(\widehat w-w^*)\|
\;\lesssim\;
\begin{cases}
n^{-\frac14}m^{-\frac{1-r}{8}} + m^{-1/4},
&
m \lesssim m_{\mathrm{crit}},
\\[6pt]
n^{-\frac{1}{2(1+\alpha)}} + m^{-1/4},
&
m \gtrsim m_{\mathrm{crit}},
\end{cases}
\]
where
\[
m_{\mathrm{crit}}
\;\asymp\;
n^{\frac{2(1-\alpha)}{(1+\alpha)(1-r)}}.
\]
\end{corollary}

\begin{proof}
Since $\gamma\ge 1$, we have
\[
\Sigma_\gamma = \Sigma+\gamma \Sigma_D \succeq \Sigma+\Sigma_D,
\]
and therefore
\[
\|(\Sigma+\Sigma_D)^{1/2}(\widehat w-w^*)\|
\;\le\;
\|\Sigma_\gamma^{1/2}(\widehat w-w^*)\|.
\]

\paragraph{Finite-sample bound.} Repeating the proof of Theorem~\ref{thm:finite-sample} with
$\Sigma^{1/2}$ replaced by $\Sigma_\gamma^{1/2}$ yields the same finite-sample bound, since both the estimation and bias terms are handled exactly as before, but using
\(
\|\Sigma_\gamma^{1/2}(\Sigma_\gamma+\lambda I)^{-1/2}\|\le 1
\)
instead of
\(
\|\Sigma^{1/2}(\Sigma_\gamma+\lambda I)^{-1/2}\|\le 1.
\)
Then, the same argument as in the proof of Corollary~\ref{thm:rates} gives
\[
\|\Sigma_\gamma^{1/2}(\widehat w-w^*)\|
\;\lesssim\;
\frac{\lambda^{-\alpha/2}}{\sqrt n}
+
\frac{\gamma^{\frac{r-1}{2}}\lambda^{-r/2}}{\sqrt n}
+
\frac{\gamma\,\lambda^{-1/2}}{\sqrt m}
+
\sqrt{\lambda},
\]
up to logarithmic factors and lower-order terms.

The optimization over $\gamma$ is identical to that of
Corollary~\ref{thm:rates}, except that we now impose the constraint $\gamma\ge 1$.
Let
\[
\gamma^\star(\lambda)
\;\asymp\;
\left(\frac{m}{n}\right)^{\!\frac{1}{3-r}}
\lambda^{\frac{1-r}{3-r}}
\]
denote the unconstrained optimizer from Corollary~\ref{thm:rates}. Under the
constraint $\gamma\ge 1$, we therefore choose
\[
\gamma(\lambda)\asymp \max\{1,\gamma^\star(\lambda)\}.
\]

\paragraph{Optimizing $\gamma$ and $\lambda$.}
When $\gamma^\star(\lambda)\ge 1$, we recover exactly the same bound as in
Corollary~\ref{thm:rates}. When $\gamma^\star(\lambda)<1$, the constraint is
active and we set $\gamma=1$, which yields the additional term
\[
\frac{\lambda^{-1/2}}{\sqrt m}.
\]
Balancing this term with the bias term $\sqrt{\lambda}$ gives
\[
\lambda \asymp m^{-1/2},
\]
and therefore an additional contribution of order $m^{-1/4}$.

Therefore, the optimal regularization parameter is
\[
\lambda
\;\asymp\;
\max\{\lambda_A^\star,\lambda_B^\star,m^{-1/2}\},
\]
where $\lambda_A^\star$ and $\lambda_B^\star$ are the optimizers obtained in
the proof of Corollary~\ref{thm:rates}. This yields
\[
\|(\Sigma+\Sigma_D)^{1/2}(\widehat w-w^*)\|
\;\lesssim\;
\begin{cases}
n^{-\frac14}m^{-\frac{1-r}{8}} + m^{-1/4},
&
m \lesssim m_{\mathrm{crit}},
\\[6pt]
n^{-\frac{1}{2(1+\alpha)}} + m^{-1/4},
&
m \gtrsim m_{\mathrm{crit}},
\end{cases}
\]
as claimed.
\end{proof}

\subsection{Adaptation to the vector-valued case}\label{sec:adaptation-vector}

The setting of this paper considers scalar-valued differential operators, for simplicity of the argument. The adaptation to vector-valued operator is straightforward as we detail in this section. 

Consider a vector-valued differential operator $D = (D_1, \dots, D_k)$, where each $D_i$ is of the form given in \Cref{as:D_order}. Then, for any $u \in \cH$ and any $x \in \cX$, $Du(x) $ is a vector in $\RR^k$. The differential data points $d_j$ are also in $\RR^k$ and we replace the risk \eqref{eq:physics-informed-emp-risk} by the following one:
\begin{equation}\label{eq:physics-informed-emp-risk-vector}
\widehat R(u)
=
\frac1n\sum_{i=1}^n\big(u(x_i)-y_i\big)^2
+\gamma\,\frac1m\sum_{j=1}^m \|Du(z_j)-d_j\|_2^2 
+\lambda\|u\|_{\mathcal H}^2.
\end{equation}

We can define component-wise feature maps $D_\ell \phi(z)$, and define the differential covariance operator as
\begin{equation}
    \Sigma_D := \mathbb E_{z \sim \rho_D}\left[\sum_{\ell=1}^k D_\ell\phi(z)\otimes D_\ell\phi(z)\right].
\end{equation}
The key point is that the analysis is carried out using covariance operators $\Sigma$ and $\Sigma_D$, and the key assumption, \Cref{ass:abstract-alignment}, is an abstract condition on $\Sigma$ and $\Sigma_D$, regardless of their expressions. As such, even if going from scalar-valued to vector-valued changes the expression of the covariance operator $\Sigma_D$ as we just saw, as long as one guarantees that \Cref{ass:abstract-alignment} is satisfied, the rest of the analysis holds. The other assumptions must be adapted in a straightforward way: the smoothness $s$ required in \Cref{as:k_Cs} is the maximum of the orders of each $D_\ell$, and for \Cref{as:bounded} we need the boundedness of each feature $\|D_\ell\phi(z)\|_{\mathcal H}$.

For the proofs, we also adapt $V_D, \widehat \Sigma_D, \widehat V_D$ by summing over $\ell$:
\[V_D   := \mathbb E\left[\sum_{\ell=1}^k d_\ell\,D_\ell\phi(z)\right]\]
\[\widehat \Sigma_D : = \frac{1}{m}\sum_{j=1}^m\sum_{\ell=1}^k D_\ell\phi(z_j)\otimes D_\ell\phi(z_j) \]
\[ \widehat V_D : = \frac{1}{m}\sum_{j=1}^m\sum_{\ell=1}^k d_{j,\ell}\,D_\ell\phi(z_j)\]
Then, the concentration inequalities are still valid and the proofs are the same.

\subsection{Learning rates with exact physical constraint}
\label{sec:oracle_physical_rates}

We derive the PIKS rate in the idealized setting where the differential information is known exactly. In this physical oracle setting, the value observations remain empirical, while the differential part of the risk is replaced by its population counterpart:
\[
\widehat R_{\mathrm{or}}(u)
=
\frac1n\sum_{i=1}^n\big(u(x_i)-y_i\big)^2
+
\gamma\,\mathbb E_{z\sim\rho_D}\big[(Du(z)-Du^*(z))^2\big]
+
\lambda\|u\|_{\mathcal H}^2 .
\]
Equivalently, in operator form, this replaces \(\widehat\Sigma_D\) and \(\widehat V_D\) by their population counterparts \(\Sigma_D\) and \(V_D\).

\begin{corollary}[Oracle learning rate]\label{cor:oracle-rate}
Assume Assumptions~\ref{as:D_order}-\ref{ass:abstract-alignment} hold, with \(r<1\). 
Then there exist choices of \(\lambda=\lambda(n)\) and \(\gamma=\gamma(n)\) such that, with high probability,
\[
\|\Sigma^{1/2}(\widehat w_{\mathrm{or}}-w^*)\|
\;\lesssim\;
n^{-\frac{1}{2(1+\alpha)}}.
\]
\end{corollary}

\begin{proof}
The oracle estimator has the closed form
\[
\widehat w_{\mathrm{or}}
=
(\widehat\Sigma+\gamma\Sigma_D+\lambda I)^{-1}
(\widehat V+\gamma V_D).
\]
Let
\[
\widehat\Sigma_{\gamma}^{\mathrm{or}}
:=\widehat\Sigma+\gamma\Sigma_D,
\qquad
\widehat V_{\gamma}^{\mathrm{or}}
:=\widehat V+\gamma V_D,
\qquad
\Sigma_\gamma:=\Sigma+\gamma\Sigma_D.
\]
Under Assumptions~\ref{as:attainability} and~\ref{as:consistency},
\[
V=\Sigma w^*,
\qquad
V_D=\Sigma_D w^*,
\qquad
V+\gamma V_D=\Sigma_\gamma w^*.
\]

The proof of Theorem~\ref{thm:finite-sample} applies verbatim, except that the differential part is no longer empirical. Hence the only stochastic deviations are
\[
\widehat\Sigma-\Sigma,
\qquad
\widehat V-V,
\]
and all terms involving \(\widehat\Sigma_D-\Sigma_D\) and \(\widehat V_D-V_D\) disappear. Up to logarithmic factors,
\[
\|\Sigma^{1/2}(\widehat w_{\mathrm{or}}-w^*)\|
\;\lesssim\;
\frac{1}{n\sqrt{\lambda}}
+
\frac{\mathcal N_V(\lambda,\gamma)}{\sqrt n}
+
\lambda\mathcal B(\lambda,\gamma),
\]
where
\[
\mathcal N_V(\lambda,\gamma)
=
\big\|(\Sigma_\gamma+\lambda I)^{-1/2}\Sigma^{1/2}\big\|_{\mathrm{HS}},
\qquad
\mathcal B(\lambda,\gamma)
=
\big\|(\Sigma_\gamma+\lambda I)^{-1/2}w^*\big\|.
\]

By Assumption~\ref{ass:abstract-alignment}, as in the proof of Corollary~\ref{thm:rates},
\[
\mathcal N_V(\lambda,\gamma)^2
\;\lesssim\;
\lambda^{-\alpha}
+
\gamma^{r-1}\lambda^{-r}.
\]
Moreover, by attainability,
\[
\lambda\mathcal B(\lambda,\gamma)\lesssim \sqrt{\lambda}.
\]
Therefore,
\[
\|\Sigma^{1/2}(\widehat w_{\mathrm{or}}-w^*)\|
\;\lesssim\;
\frac{1}{n\sqrt{\lambda}}
+
\frac{\lambda^{-\alpha/2}}{\sqrt n}
+
\frac{\gamma^{\frac{r-1}{2}}\lambda^{-r/2}}{\sqrt n}
+
\sqrt{\lambda}.
\]

Since the differential information is known at the population level, there is no derivative-sample variance term increasing with \(\gamma\). We choose \(\gamma\) large enough so that
\[
\gamma^{r-1}\lambda^{-r}
\;\lesssim\;
\lambda^{-\alpha}.
\]
Equivalently, when \(r>\alpha\), it suffices to take
\[
\gamma
\;\gtrsim\;
\lambda^{-\frac{r-\alpha}{1-r}},
\]
while when \(r\le\alpha\), any \(\gamma\gtrsim1\) is enough. With this choice,
\[
\|\Sigma^{1/2}(\widehat w_{\mathrm{or}}-w^*)\|
\;\lesssim\;
\frac{1}{n\sqrt{\lambda}}
+
\frac{\lambda^{-\alpha/2}}{\sqrt n}
+
\sqrt{\lambda}.
\]

Balancing the dominant variance term with the bias,
\[
\frac{\lambda^{-\alpha/2}}{\sqrt n}
=
\sqrt{\lambda},
\]
gives
\[
\lambda\asymp n^{-\frac{1}{1+\alpha}}.
\]
For this choice,
\[
\frac{\lambda^{-\alpha/2}}{\sqrt n}
\asymp
\sqrt{\lambda}
\asymp
n^{-\frac{1}{2(1+\alpha)}},
\]
while
\[
\frac{1}{n\sqrt{\lambda}}
=
n^{-1+\frac{1}{2(1+\alpha)}}
\]
is lower order since \(\alpha>0\). Hence
\[
\|\Sigma^{1/2}(\widehat w_{\mathrm{or}}-w^*)\|
\;\lesssim\;
n^{-\frac{1}{2(1+\alpha)}}.
\]
\end{proof}

\section{Examples }
We now illustrate Assumption~\ref{ass:abstract-alignment} on concrete examples where the decomposition can be computed explicitly.
\subsection{Partial Laplacian and periodic Sobolev RKHS}

\begin{lemma}[Capacity decomposition for Sobolev spaces]
Let $\mathcal X=\mathbb T^d $, let $k$ be a Matérn kernel of smoothness $s>d/2$, and let $D=\sum_{i\in S}\partial_{x_i}^2$ be a partial Laplacian with $S\subseteq\{1,\dots,d\}$.
Assume that value and differential samples are drawn uniformly on $\mathbb T^d$.

Then the value covariance operator $\Sigma$ admits a decomposition $\Sigma=\Sigma_1+\Sigma_2$ satisfying Assumption~\ref{ass:abstract-alignment}, with exponents
\[
\alpha=\frac{d-|S|}{2s},
\qquad
r=\frac{d}{2s}.
\]
\end{lemma}

\begin{proof} 

\textbf{Step 1: diagonalization in the Fourier basis.}
Let $D$ be a constant-coefficient differential operator of order $q$, of the form
\[
D = \sum_{|\alpha|\le q} c_\alpha\,\partial^\alpha,
\qquad
\partial^\alpha = \partial_{x_1}^{\alpha_1}\cdots\partial_{x_d}^{\alpha_d},
\]
with real coefficients $c_\alpha$. Its Fourier symbol is given by
\[
P(k) = \sum_{|\alpha|\le q} c_\alpha\,(2\pi i k)^\alpha,
\qquad k\in\mathbb Z^d,
\]
so that $D e^{2\pi i k\cdot x} = P(k)\,e^{2\pi i k\cdot x}$.

Since the kernel $k$ is translation invariant and both $x$ and $z$ are sampled uniformly on $\mathbb T^d$, the covariance operators
\[
\Sigma=\mathbb E[\phi(x)\otimes \phi(x)],\qquad 
\Sigma_D=\mathbb E[D\phi(z)\otimes D\phi(z)]
\]
are convolution operators and are therefore diagonal in the Fourier basis
$\{e_k(x):=e^{2\pi i k\cdot x}\}_{k\in\mathbb Z^d}$.
The corresponding eigenvalues are
\[
\sigma_k=\mu_k,\qquad 
\tau_k = |P(k)|^2\,\mu_k .
\]

\textbf{Step 2: partial Laplacian and visible/invisible frequencies.}
For $D=\sum_{i\in S}\partial_{x_i}^2$, the symbol is
\[
P(k)=-(2\pi)^2\|k_S\|^2,
\qquad\text{so}\qquad
\tau_k=(2\pi)^4\|k_S\|^4\,\mu_k,
\]
where $k_S$ denotes the restriction of $k$ to coordinates in $S$.
Define the index sets
\[
\mathcal Z:=\{k\in\mathbb Z^d:\ k_S=0\},\qquad \mathcal Z^c:=\mathbb Z^d\setminus \mathcal Z.
\]
Thus $\mathcal Z$ consists of frequencies \emph{invisible} to $D$ (since $P(k)=0$), whereas $\mathcal Z^c$ corresponds to frequencies \emph{detectable} by $D$.

\textbf{Step 3: the decomposition $\Sigma=\Sigma_1+\Sigma_2$.}
Define
\[
\Sigma_1:=\sum_{k\in\mathcal Z}\mu_k\, e_k\otimes e_k,
\qquad
\Sigma_2:=\sum_{k\in\mathcal Z^c}\mu_k\, e_k\otimes e_k.
\]
Then $\Sigma=\Sigma_1+\Sigma_2$ and $\Sigma_1,\Sigma_2\succeq 0$.

\paragraph{Step 4: capacity of the invisible component $\Sigma_1$.}
Let $d_0=d-|S|$ and write $k_{-S}\in\mathbb Z^{d-S}$ for the subvector of coordinates outside $S$. The map $k_{-S}\mapsto (k_{-S},0_S)$ is a bijection $\mathbb Z^{d-S}\to\mathcal Z$ and
$|(k_{-S},0_S)|=|k_{-S}|$, hence
\[
\sum_{k\in\mathcal Z}\frac{\mu_k}{\mu_k+\lambda}
=\sum_{k_{-S}\in\mathbb Z^{d-S}}
\frac{\mu_{(k_{-S},0_S)}}{\mu_{(k_{-S},0_S)}+\lambda},
\qquad
\mu_{(k_{-S},0_S)}\asymp (1+|k_{-S}|^2)^{-s}.
\]
Therefore, by a standard comparison between lattice sums and integrals, for all $\lambda\in(0,1]$,
\[
\Tr\!\big(\Sigma_1(\Sigma_1+\lambda I)^{-1}\big)
=\sum_{k\in\mathcal Z}\frac{\mu_k}{\mu_k+\lambda}
\;\lesssim\;\lambda^{-\frac{d_0}{2s}}.
\]
This verifies the first trace condition with exponent $\alpha=d_0/(2s)$.

\textbf{Step 5: alignment of $\Sigma_2$ with $\Sigma_D$.}
On $\mathcal Z^c$ we have $\|k_S\|\ge 1$ and hence $\tau_k=(2\pi)^4\|k_S\|^4\mu_k \ge c\,\mu_k$ for some $c>0$.
Therefore, for all $t\in(0,1]$,
\[
\Tr\!\big((\Sigma_D+tI)^{-1}\Sigma_2\big)
=\sum_{k\in\mathcal Z^c}\frac{\mu_k}{\tau_k+t}
\;\le\;\sum_{k\in\mathcal Z^c}\frac{\mu_k}{c\mu_k+t}.
\]
Using $\mu_k\asymp (1+|k|^2)^{-s}$ and standard comparison between lattice sums and integrals gives
\[
\sum_{k\in\mathbb Z^d}\frac{\mu_k}{\mu_k+t}
\;\lesssim\; t^{-\frac{d}{2s}},
\qquad t\in(0,1],
\]
and hence
\[
\Tr\!\big((\Sigma_D+tI)^{-1}\Sigma_2\big)\;\lesssim\; t^{-\frac{d}{2s}}.
\]
This verifies the second trace condition with exponent $r=d/(2s)$.

Combining Steps 3-5 yields Assumption~\ref{ass:abstract-alignment} with the claimed exponents.
\end{proof}
\subsection{Gradients}

Let $\Omega\subset\R^d$ be a bounded, connected, Lipschitz domain.
Fix $s>\frac d2 +1$. Let $\bmH = H^s(\Omega)$, with norm $\| \cdot \|_\bmH$ equivalent to the standard $H^s(\Omega)$ norm. We denote by $K$ the reproducing kernel associated to $\bmH$, and by $\phi : \Omega \to \cH$ the feature map defined, for all $x \in \Omega$, by $\phi(x) = K(x, \cdot) \in \cH$.
Choose $\rho = \rho_D$ the uniform distribution on $\Omega$. 

\subsubsection{Covariance operator decomposition}

\paragraph{Definition of $\Sigma$}Define the covariance operator 
\[\Sigma := \EE_{x \sim \rho}[\phi(x) \otimes \phi(x)].\]
We have, for all $u\in\bmH$,
\begin{equation}
\label{eq:Sigma_form}
\langle \Sigma u, u\rangle_\bmH = \|u\|_{L^2(\rho)}^2.
\end{equation}

\paragraph{Definition of $\Sigma_D$} Define the gradients covariance operator
\[\Sigma_D = \EE_{z \sim \rho_D}\left[\sum_{i = 1}^d \partial_i \phi(z) \otimes \partial_i \phi(z)\right].\]
We have, for all $u \in \cH$, 
\begin{equation}\label{eq:SigmaD_form}\langle \Sigma_D u , u \rangle_\cH = \| \nabla u \|_{L^2(\rho)}^2.
\end{equation}

\paragraph{Definition of $\Sigma_1$} Denote by \(\phi_\Omega = \EE_{x \sim \rho}[\phi(x)] \in \bmH\) the representer of the averaging functional, so that
\[
\langle u,\phi_\Omega\rangle_{\bmH}
=
\EE_{x \sim \rho}[ u(x)] =: u_\Omega.
\]
Denote
\[ \Sigma_1 = \phi_\Omega \otimes \phi_\Omega.\]
Then for all $u\in\bmH$,
\begin{equation}
\label{eq:Sigma1_form}
\langle \Sigma_1u , u \rangle_\bmH =  u_\Omega^2.
\end{equation}

\paragraph{Definition of $\Sigma_2$} Next, define the bounded operator
\[
A:\bmH\to L^2(\rho),
\qquad
Au := u - u_\Omega,
\]
Set
\[
\Sigma_2 := A^*A,
\]
so that for all $u\in\bmH$,
\begin{equation}
\label{eq:Sigma2_form}
\langle \Sigma_2u, u \rangle_\bmH =  \|u - u_\Omega\|_{L^2(\rho)}^2 = \Var_{x \sim \rho}(u(x)).
\end{equation}

\begin{proposition}[Decomposition of $\Sigma$]
\label{prop:Sigma_decomposition}
One has
\[
\Sigma = \Sigma_1+\Sigma_2.
\]
\end{proposition}

\begin{proof}
Let \(u\in\bmH\). By the variance decomposition formula,
\[
\EE_{x\sim\rho}\big[u(x)^2\big]
=
\Big(\EE_{x\sim\rho}[u(x)]\Big)^2
+
\Var_{x\sim\rho}(u(x)).
\]
Equivalently,
\[
\|u\|_{L^2(\rho)}^2
=
u_\Omega^2
+
\|u-u_\Omega\|_{L^2(\rho)}^2
=
u_\Omega^2+\Var_{x\sim\rho}(u(x)).
\]
Using \eqref{eq:Sigma_form}, \eqref{eq:Sigma1_form}, and \eqref{eq:Sigma2_form}, we obtain
\[
\langle \Sigma u,u\rangle_{\bmH}
=
\langle \Sigma_1u,u\rangle_{\bmH}
+
\langle \Sigma_2u,u\rangle_{\bmH}
\]
for every \(u\in\bmH\).

Since \(\Sigma\), \(\Sigma_1\), and \(\Sigma_2\) are bounded self-adjoint operators on \(\bmH\), equality of their quadratic forms implies, by polarization, that
\[
\Sigma=\Sigma_1+\Sigma_2.
\]
\end{proof}

\subsubsection{Computation of the coefficients}

Let us first bound the effective dimension
\[
\mathcal{N}_1(\lambda):=\Tr\big(\Sigma_1(\Sigma_1+\lambda I)^{-1}\big).
\]

Since \(\Sigma_1\) has rank one, the effective dimension $\mathcal{N}_1(\lambda)$ is uniformly bounded. In particular, it satisfies
\[
\mathcal{N}_1(\lambda)\le C_\alpha \lambda^{-\alpha},
\qquad \lambda\in(0,1],
\]
with
\[
\boxed{\alpha=0}.
\]

Let us now bound
\[
\Tr\!\big((\Sigma_D+\lambda I)^{-1}\Sigma_2\big).
\]

Since \(\Omega\) is bounded, connected, and Lipschitz, the Poincaré--Wirtinger inequality \citep{evans2010partial} holds: there exists \(C>0\) such that, for every
\(u\in H^1(\Omega)\),
\[
\|u-u_\Omega\|_{L^2(\rho)}
\le
C\|\nabla u\|_{L^2(\rho)}.
\]
Therefore, for every \(u\in\bmH\),
\[
\begin{aligned}
\langle \Sigma_2u,u\rangle_{\bmH}
&=
\|u-u_\Omega\|_{L^2(\rho)}^2 \\
&\le
C^2\|\nabla u\|_{L^2(\rho)}^2 \\
&=
C^2\langle \Sigma_Du,u\rangle_{\bmH}.
\end{aligned}
\]
Equivalently, in the Loewner order,
\[
0\preceq \Sigma_2\preceq C^2\Sigma_D.
\]
It follows that
\[
\begin{aligned}
\Tr\!\big((\Sigma_D+\lambda I)^{-1}\Sigma_2\big)
&\le 
C^2
\Tr\!\big(\Sigma_D(\Sigma_D+\lambda I)^{-1}\big).
\end{aligned}
\]

We now bound the last trace. Let \((\eta_j)_{j\ge1}\) be the nonzero eigenvalues of
\(\Sigma_D\), arranged in nonincreasing order. Since \(\Sigma_D\) is the
covariance operator associated with the map
\[
u\longmapsto \nabla u,
\qquad
H^s(\Omega)\to L^2(\rho;\mathbb{R}^d),
\]
the standard eigenvalue estimate for Sobolev embeddings on bounded Lipschitz
domains gives
\[
\eta_j \le C_D j^{-2(s-1)/d}.
\]
Set
\[
\beta:=\frac{2(s-1)}{d}.
\]
Because \(s>\frac d2+1\), we have \(\beta>1\). Hence, for \(0<\lambda\le1\),
\[
\begin{aligned}
\Tr\!\big(\Sigma_D(\Sigma_D+\lambda I)^{-1}\big)
&=
\sum_{j\ge1}\frac{\eta_j}{\eta_j+\lambda} \\
&\le
\sum_{j\ge1}\frac{C_D j^{-\beta}}{C_D j^{-\beta}+\lambda} \\
&\lesssim
\lambda^{-1/\beta}.
\end{aligned}
\]
Since \(1/\beta=d/(2(s-1))\), we obtain
\[
\Tr\!\big((\Sigma_D+\lambda I)^{-1}\Sigma_2\big)
\lesssim
\lambda^{-d/(2(s-1))}.
\]
Thus one may take
\[
\boxed{
r=\frac{d}{2(s-1)}.
}
\]
In particular, since \(s>\frac d2+1\), one has \(r<1\).

\section{Detailed experimental setup}\label{sec:app-experiments}

In this section we provide more details about the setup used for the experiments of \cref{sec:experiments}.

\paragraph{Implementation}
To maximize the flexibility of our system, we implemented the PIKS estimator using the \href{https://docs.jax.dev}{Jax framework}, which allows to efficiently compute arbitrary derivatives.
While very flexible, any automatic differentiation framework does introduce some computational overhead and requires particular care when implementing kernel functions. For example, Mat\'ern kernels require differentiating through a square-root which is numerically unstable at 0. We introduce a small additive nugget term to ensure stability everywhere.
The kernel solver is a straightforward implementation of the equations in \cref{lem:piks-finite-dimensional} which yield a space complexity of $O((n+m)^2)$ and time complexity of $O((n+m)^3)$.

\paragraph{Bounded domain gradient example}
We repeated the experiment 10 times to obtain standard deviations reported in \cref{fig:inc-m}. We set the variance on the differential data to be low in order to avoid having to scale $m$ too much before seeing a saturation effect. All data points were uniformly sampled on a unit disk with 1000 samples used for validation (to select hyperparameters) and 10000 for computing the test error.

\paragraph{Partial Laplacian example}
We give a few more details on the kernel and target functions for the example on Sobolev spaces on the torus with the partial Laplacian operator. Here we used $\mathcal X = \mathbb T^4$. We sampled uniformly at random $F=64$ frequency vectors $k_\ell\in\mathbb{Z}^4$ such that for $\ell\in[0, 15]$, $k_\ell \sim [\mathcal U([-8, 8] \backslash \{0\}), 0, 0, 0]$; for $\ell \in [16, 31]$, $k_\ell \sim [0, \mathcal U([-8, 8]  \backslash \{0\})]$ and so on (only one dimension active for each vector).
For feature map $\phi_\ell(x) = \sqrt{2} \cos(2\pi\langle k_\ell, x\rangle)$, target function and kernel are defined as
\begin{equation*}
    u^*(x) = u_0 + \sum_{\ell=1}^Fc_\ell \phi_\ell(x), \qquad k(x, x') = 1 + \sum_{\ell=1}^F\mu_\ell\phi_\ell(x)\phi_\ell(x'),
\end{equation*}
with the smoothness coefficients
\begin{equation*}
    \mu_\ell \asymp (1 + \|k_\ell\|^2)^{-\beta}, \quad c_\ell \asymp (1 + \|k_\ell\|^2)^{-\delta}, \quad \beta > 1, \quad \delta > \frac{1}{2}.
\end{equation*}
In particular, we set $\beta = \delta = 4.1$ to have a well-specified problem in a smooth enough space.
We use additive Gaussian noise, with standard-deviation equal to 10\% to the range of the data. This was done because the different data components have vastly different numerical ranges, and setting a fixed variance would have introduced artifacts in the results.

The best-fit lines of \cref{fig:part-lap} are obtained by least-squares regression of the experimental data against the function $an^{b}$ to find the exponential coefficient $b$. Then, using \cref{thm:rates}, we get $\alpha = -\frac{1}{2b} - 1$.

\paragraph{Hyperparameters}
The two hyperparameters of PIKS ($\lambda$ and $\gamma$), as well as the hyperparameters of the kernel (notably the length-scale of the Mat\'ern kernel) were determined by a coarse grid-search using small, noisy validation sets to estimate their performance at generalization time.

\end{document}